\documentclass{amsart}
\usepackage{amsfonts}
\usepackage{amsmath}
\usepackage{enumerate}
\usepackage{tcolorbox}
\usepackage{amsmath,amssymb}

\usepackage{mathrsfs} 

\usepackage{algorithm,algpseudocode}
\usepackage{setspace} 

\usepackage{mathtools} 

\usepackage{booktabs} 
\usepackage{array} 
\usepackage{hyperref}
\hypersetup{
    colorlinks,
    citecolor=black,
    filecolor=black,
    linkcolor=black,
    urlcolor=black
}

\NewEnviron{mblock}[1][colback=yellow]{%
\begin{tcolorbox}
\BODY
\end{tcolorbox}
}

\tcbuselibrary{most}
\newtcolorbox{mybox}[2][]
{colback = red!5!white, colframe = red!75!black, fonttitle = \bfseries,
colbacktitle = red!85!black, enhanced,
attach boxed title to top center={yshift=-2mm},
title=#2,#1}

\usepackage{multirow}

\newtheorem{theorem}{Theorem}[section]
\newtheorem{lemma}[theorem]{Lemma}

\newtheorem{definition}[theorem]{Definition}
\newtheorem{corollary}[theorem]{Corollary}

\newtheorem{assumption}[theorem]{Assumption}

\numberwithin{equation}{section}

 \def\M{\mathcal{M}} 
\def\Z{\mathcal{Z}} 
\def\lan{\operatorname{lan}}

\title{Landmark Alternating Diffusion}

\author[S.-Y. Yeh]{Sing-Yuan Yeh}
\address{Data Science Degree Program, National Taiwan University and Academia Sinica, Taipei, 106, Taiwan}
\email{d10948003@ntu.edu.tw}

\author[H.-T. Wu]{Hau-Tieng Wu}
\address{Courant Institute of Mathematical Sciences, New York University, New York, NY, 10012 USA}
\email{hauwu@cims.nyu.edu}

\author[R. Talmon]{Ronen Talmon}
\address{Viterbi Faculty of Electrical and Computer Engineering, Technion - Israel Institute of Technology, Haifa, 3200003, Israel}
\email{ronen@ef.technion.ac.il}

\author[M.-P. Tsui]{Mao-Pei Tsui}
\address{Department of Mathematics, National Taiwan University, Taipei, 106, Taiwan;
National Center for Theoretical Sciences, Mathematics Division, Taipei, 106, Taiwan}
\email{maopei@math.ntu.edu.tw}

\begin{document}

\begin{abstract}
Alternating Diffusion (AD) is a commonly applied diffusion-based sensor fusion algorithm. While it has been successfully applied to various problems, its computational burden remains a limitation. Inspired by the landmark diffusion idea considered in the Robust and Scalable Embedding via Landmark Diffusion (ROSELAND), we propose a variation of AD, called Landmark AD (LAD), which captures the essence of AD while offering superior computational efficiency. We provide a series of theoretical analyses of LAD under the manifold setup and apply it to the automatic sleep stage annotation problem with two electroencephalogram channels to demonstrate its application.
\end{abstract}

\maketitle


\section{Introduction}

With the advancement of sensing technology, it is becoming increasingly common to collect data from multiple sensors simultaneously. Sensor fusion is a technique that involves combining data from multiple sensors to obtain an accurate, reliable, and comprehensive understanding of the underlying phenomenon or environment being monitored \cite{7214350,gustafsson2012statistical}. This process aims to overcome limitations inherent in the usage of individual sensors, such as noise, bias, or limited coverage, by leveraging complementary information from multiple sources. 
Various sensor fusion techniques and algorithms exist, including statistical methods, machine learning algorithms, Bayesian inference, and mathematical modeling. 
Specific examples include linear algorithms such as
the canonical correlation analysis (CCA) and its variants \cite{hotelling1936relations,Horst1961,6788402,Hwang2013}, nonparametric canonical correlation analysis (NCCA) \cite{NCCApaper}, alternating diffusion (AD) \cite{lederman2018learning, talmon2019} and its generalizations \cite{katz2019alternating,shnitzer2019recovering,shnitzer2024spatiotemporal}, time-coupled diffusion maps \cite{marshall2018time}, and multiview diffusion maps \cite{lindenbaum2020multi}, to mention just a few. 
See \cite{shnitzer2019recovering,zhuang2020technical,shnitzer2024spatiotemporal} and the references therein for a more comprehensive literature review.
Some of these algorithms aim at extracting common information shared by the datasets acquired by the different sensors while others attempt to identify the different and unique information among them.

Our focus in this paper is centered on the diffusion-based approach, particularly the AD algorithm. 
AD has shown remarkable efficacy in solving a diverse array of problems, producing significant results across various fields. These applications include, but are not limited to, audio-visual voice activity detection \cite{8281539}, sequential audio-visual correspondence \cite{7558246}, automatic sleep stage identification using multiple sensors \cite{lederman2015}, and advanced interpretation of electroencephalogram (EEG) signals \cite{liu2020diffuse,liu2020explore}. Furthermore, AD has been instrumental in seismic event modeling \cite{lindenbaum2018multiview}, predicting Intelligence Quotient (IQ) using two functional Magnetic Resonance Imaging (fMRI) paradigms \cite{xiao2019manifold}, multi-microphone speaker localization \cite{laufer2020data}, fetal electrocardiogram analysis \cite{shnitzer2019recovering}, and the development of multiview spectral clustering methods \cite{roman2023multiview}. This diverse range of applications highlights AD's versatility and its potential for driving significant advancements in various scientific and technological domains.
Additionally, various theoretical results have been established to support the efficacy of AD. For example, in \cite{lederman2018learning}, it is shown that AD extracts the common information under the general metric space model. Under the manifold setup, it is demonstrated in \cite{talmon2019} that AD asymptotically converges to the Laplace-Beltrami operator of the common manifold, which is insusceptible to the sensor-specific nuisance information modeled by a product manifold structure. The Laplace-Beltrami operator might be deformed if the common information shared by the two sensors is not identical but diffeomorphically related. The behavior of AD under the null hypothesis has been reported in \cite{DW1} using random matrix theory and free probability framework, showing that asymptotically the spectrum converges to a deterministic distribution. Under the same framework as in \cite{DW1}, various noisy setups under the nonnull hypothesis have been studied in \cite{ding2024kernel}, including scenarios that model situations when one sensor malfunctions. 

Although AD has been extensively researched and applied in various fields, a notable limitation arises from its computational demands. This is primarily due to the eigendecomposition, which is a critical yet resource-intensive step in the algorithm. In this paper, we present a new variation of AD that is theoretically sound and retains the core principles of the original method. This variation is designed to significantly improve computational efficiency. Our approach makes use of the landmark diffusion concept originally introduced in the algorithm \emph{Robust and Scalable Embedding via Landmark Diffusion} (ROSELAND) \cite{shen2020,shen2022} to expedite the diffusion maps algorithm, which is a central manifold learning method \cite{coifman2006}. The concept of landmark diffusion involves constraining diffusion from one point to another using a pre-designed landmark set. By leveraging this landmark set, one can alleviate the computational demand of the eigendecomposition of a large kernel matrix of the entire dataset using the eigendecomposition of the much smaller kernel matrix of the landmark set. Specifically, by forcing the diffusion on each sensor to go through a predetermined landmark set, the AD algorithm can be sped up. We refer to the proposed algorithm as \emph{Landmark Alternating Diffusion} (LAD). Additionally, besides demonstrating its computational superiority, we provide a series of theoretical analyses under the manifold framework. Furthermore, we showcase the application of LAD to the automatic sleep stage annotation problem, illustrating its practical utility.

The remainder of the paper is organized as follows. In Section \ref{review section for AD}, we review AD. In Section \ref{sec:mainalg}, we present the proposed LAD algorithm. Justification of LAD on the linear algebra level is given in Section \ref{section justification of LAD on linear algebra level}. In Section \ref{sec:main_thm}, we analyze the asymptotic behavior of LAD under a manifold setting, which equates LAD and AD on the asymptotical region. This analysis also provides us with a method to choose the hyperparameters of LAD. In Section \ref{sec:simu}, we apply Landmark AD to several simulation datasets to validate the theory proposed in Section \ref{sec:main_thm}. In Section \ref{sec:sleep}, LAD is applied to the sleep dataset by fusing signals from two channels to analyze sleep stages. In Section \ref{sec:proof}, we provide the proof of the main theorem in Section \ref{sec:main_thm}.



\section{Background on Alternating Diffusion (AD)} \label{review section for AD}
We start with a quick review of the AD algorithm. Suppose we have $n$ pairs of measurements recorded from two sensors simultaneously, denoted as $\mathcal{X}:=\{(s_i,r_i)\}_{i=1}^n\subset(\mathfrak{M}_1,d_1)\times(\mathfrak{M}_2,d_2)$, where $(\mathfrak{M}_1,d_1)$ and $(\mathfrak{M}_2,d_2)$ are metric spaces; that is, the dataset $\mathcal{X}_1:=\{s_i\}_{i=1}^p\subset \mathfrak{M}_1$ is from the first sensor, and $\mathcal{X}_2:=\{r_i\}_{i=1}^p\subset \mathfrak{M}_2$ is from the second sensor. Take two non-negative smooth kernel functions $\tilde{K}^{(1)},\,\tilde{K}^{(2)}: \mathbb{R}_{\geq 0} \rightarrow \mathbb{R}_{+}$ that decay sufficiently fast. To simplify the discussion, we assume $\tilde{K}^{(1)}$ and $\tilde{K}^{(2)}$ are Gaussian functions $\frac{1}{\sqrt{2 \pi}} e^{-t^2 / 2}$, but other kernels can be considered. Define the affinity matrices $\mathbf{W}^{(1)},\,\mathbf{W}^{(2)}\in \mathbb{R}^{n\times n}$ for the first and second sensors as
\begin{equation*}
    \mathbf{W}^{(1)}_{ij}=\tilde{K}^{(1)}\left(\frac{d^{(1)}(s_i,s_j)^2}{\epsilon_1}\right)\ \ \mbox{and}\ \  \mathbf{W}^{(2)}_{ij}=\tilde{K}^{(2)}\left(\frac{d^{(2)}(r_i,r_j)^2}{\epsilon_2}\right)\,,
    \end{equation*}
where $\epsilon_1, \epsilon_2>0$ are kernel bandwidths that can be chosen by the user. Define the associated diagonal degree matrices by 
\begin{equation*}
    \mathbf{D}^{(\ell)}_{ii}=\sum_{j=1}^n \mathbf{W}_{ij}^{(\ell)}
\end{equation*}
where $\ell=1,2$, and the diffusion matrices 
\begin{equation*}
\mathbf{M}^{(\ell)}={\mathbf{D}^{(\ell)}}^{-1}\mathbf{W}^{(\ell)}\,.
\end{equation*}
The \textit{alternating diffusion matrix} introduced in \cite{lederman2018learning,talmon2019} is given by
\begin{equation}\label{eq:ad_matrix}
\mathbf{M}=\mathbf{M}^{(1)}\mathbf{M}^{(2)}\in \mathbb{R}^{n\times n}\,.
\end{equation}
$\mathbf{M}$ is also a diffusion matrix in the sense that the sums of rows equal $1$ and the spectral norm is bounded by $1$, but in general, it might not be diagonalizable. 

Suppose the eigenvalue decomposition (EVD) of $\mathbf{M}=\mathbf{U}\mathbf{\Lambda} \mathbf{U}^{-1}$ exists, where $\mathbf{U}\in Gl(n)$ and $\mathbf{\Lambda}$ is a diagonal matrix. Denote the eigenvalues satisfying $1=\lambda_1\geq\lambda_2\geq\cdots\geq \lambda_n$ and the associated eigenvectors $u_1,u_2,\ldots,u_n$. Choose $1\leq q\leq n$, and denote $\mathbf{U}_q=[u_2,\ldots,u_{q+1}]\in \mathbb{R}^{n\times q}$ and corresponding eigenvalues $\mathbf{\Lambda}_q=\texttt{diag}(\lambda_2,\lambda_3,\ldots,\lambda_{q+1})$. We then embed the dataset into a $q$-dim Euclidean space via $\Psi_{t}$ by
\begin{equation*}
\Psi_{t}:\left(s_i,r_i\right)\mapsto e^\top_i\mathbf{U}_q\mathbf{\Lambda}_q^t\in\mathbb{R}^q
\end{equation*}
where $e_j$ is a unit column vector with 1 in the $j$-th entry, $\top$ means the transpose and $t>0$ is the diffusion time chosen by the user. $\Psi_{t}$ is the AD {\em with the initial diffusion on the second sensor}, or in short the AD. While the AD has been applied to a broad range of problems, e.g.,  \cite{lederman2015,katz2019alternating,liu2020diffuse,laufer2020data}, it is limited by its computational time, which is governed by the eigenvalue decomposition. Note that even if $\mathbf{M}^{(1)}$ and $\mathbf{M}^{(2)}$ are both sparse, which could make the eigenvalue decomposition of each more efficient, the product in \eqref{eq:ad_matrix} hinders the sparsity structure.

\section{Proposed Landmark Alternating Diffusion (LAD) Algorithm}\label{section LAD}
In this section, we present the proposed LAD algorithm with initial theoretical support from a linear algebra viewpoint. 

\subsection{Landmark alternating diffusion}\label{sec:mainalg}
To alleviate the computational burden and speed up AD, inspired by a recent landmark diffusion algorithm ROSELAND \cite{shen2020,shen2022}, we propose a new algorithm, called {\em landmark AD} (LAD). 

Take a set $\tilde{\mathcal{Z}}:=\{(a_k,b_k)\}_{k=1}^m\subset (\mathfrak{M}_1,d_1)\times  (\mathfrak{M}_2,d_2)$, where $m<n$. We call $\tilde{\mathcal{Z}}$ the {\em landmark set}, $\tilde{\mathcal{Z}}_1:=\{a_k\}_{k=1}^m\subset (\mathfrak{M}_1,d_1)$ the landmark set of the first sensor, and $\tilde{\mathcal{Z}}_2:=\{b_k\}_{k=1}^m\subset (\mathfrak{M}_2,d_2)$ the landmark set of the second sensor. In practice, $\tilde{\mathcal{Z}}$ could be a subset of $\tilde{\mathcal{X}}$. Construct two \textit{landmark affinity matrices} $\bar{\mathbf{W}}^{(1)},\,\bar{\mathbf{W}}^{(2)}\in \mathbb{R}^{n\times m}$, which are defined by
\begin{equation}
\bar{\mathbf{W}}^{(1)}_{ik}=\tilde{K}^{(1)}\left(\frac{d^{(1)}(s_i,a_k)^2}{\epsilon_1}\right)\ \ \mbox{and}\ \ \bar{\mathbf{W}}^{(2)}_{ik}=\tilde{K}^{(2)}\left(\frac{d^{(2)}(r_i,b_k)^2}{\epsilon_2}\right)\,.
\end{equation}
Construct a diagonal matrix $\bar{\mathbf{D}}^{(2)}$ as follows:
\begin{equation}
\bar{\mathbf{D}}^{(2)}=\texttt{diag}\big(\bar{\mathbf{W}}^{(2)\top}\bar{\mathbf{W}}^{(2)}\mathbf{1}_{m}\big)\,,
\end{equation}
where $\mathbf{1}_{m}\in \mathbb{R}^{m\times 1}$ is a column vector with all entries 1. Then, fix a {\em normalization factor} $\alpha\in[0,1]$ and define the \textit{landmark diffusion matrix} of the second sensor 
\begin{equation}
\bar{\mathbf{M}}^{(2)}_\alpha:=\bar{\mathbf{W}}^{(2)}(\bar{\mathbf{D}}^{(2)})^{-\alpha}\in \mathbb{R}^{n\times m}\,.
\end{equation}
Next, construct a diagonal matrix $\bar{\mathbf{D}}^{(1)}_\alpha$ as \begin{equation}
\bar{\mathbf{D}}^{(1)}_{\alpha}=\texttt{diag}\big(\bar{\mathbf{W}}^{(1)}\bar{\mathbf{M}}^{(2)\top}_\alpha\mathbf{1}_{n}\big)\in \mathbb{R}^{n\times n}\,.
\label{eq:diag1}
\end{equation}
Then, the \textit{landmark diffusion matrix} of the first sensor are defined as
\begin{equation}
\bar{\mathbf{M}}^{(1)}_\alpha=(\bar{\mathbf{D}}^{(1)}_\alpha)^{-1}\bar{\mathbf{W}}^{(1)}\in \mathbb{R}^{n\times m}\,.
\label{eq:markov1}
\end{equation}
The \textit{landmark AD matrix} is defined as 
\[
\bar{\mathbf{M}}_\alpha=\bar{\mathbf{M}}^{(1)}_\alpha\bar{\mathbf{M}}^{(2)\top}_\alpha\in \mathbb{R}^{n\times n}\,.
\]

To speed up the computation, consider the matrix $\bar{\mathbf{M}}^{(2)\top}_\alpha\bar{\mathbf{M}}^{(1)}_\alpha\in \mathbb{R}^{m\times m}$, compute the eigenstructure of $\bar{\mathbf{M}}^{(2)\top}_\alpha\bar{\mathbf{M}}^{(1)}_\alpha$ on $m$ landmarks and extrapolate to the whole database in the following way. 
Denote the EVD of $\bar{\mathbf{M}}^{(2)}_\alpha\bar{\mathbf{M}}^{(1)}_\alpha$ by
\begin{equation}
\bar{\mathbf{M}}^{(2)\top}_\alpha\bar{\mathbf{M}}^{(1)}_\alpha=\bar{\mathbf{V}}\bar{\mathbf{\Lambda}}\bar{\mathbf{V}}^{-1}\in \mathbb{R}^{m\times m}\,,
\end{equation}
where the diagonal entries of $\bar{\mathbf{\Lambda}}$ are $\bar{\lambda}_1\geq\bar{\lambda}_2\geq\cdots\geq \bar{\lambda}_{m'}>\bar{\lambda}_{m'+1}=\ldots=\bar{\lambda}_{m}=0$. Set $\bar{\mathbf{U}}=\bar{\mathbf{M}}^{(1)}_\alpha\bar{\mathbf{V}}\in \mathbb{R}^{n\times m}$. Choose $1\leq q\leq m'$. Let $\bar{\mathbf{U}}_q\in \mathbb{R}^{n\times q}$ be a matrix that contains the first $q$ columns of $\bar{\mathbf{U}}$ and $\bar{\mathbf{\Lambda}}_q$ a diagonal matrix containing the associated eigenvalues. The $\alpha$-normalized LAD {\em with the initial diffusion on the second sensor}, or $\alpha$-LAD in short, is defined as
\begin{equation*}
\bar{\Psi}_{t}:\left(s_i,r_i\right)\mapsto e^\top_i \bar{\mathbf{U}}_q\,(\bar{\mathbf{\Lambda}}_q)^t\in\mathbb{R}^q\,,
\end{equation*}
where $t>0$ is the diffusion time  chosen by the user. The output is the diffusion result on $\{r_i\}_{i=1}^n$. The pseudocode of $\alpha$-LAD is summarized in Algorithm \ref{alg:lead}.

\begin{algorithm}[ht]
\caption{Landmark Alternating Diffusion (LAD) }
\begin{flushleft}
\hspace*{\algorithmicindent} \textbf{Input:} Two datasets, $\{s_i\}_{i=1}^n\subset (\mathfrak{M}_1,d_1)$ and $\{r_i\}_{i=1}^n\subset (\mathfrak{M}_2,d_2)$, are sampled simultaneously from two sensors. Choose $\epsilon_1,\epsilon_2>0$, embedding dimension $0<q< m$, diffusion time $t>0$, $\alpha\in[0,1]$ and landmark size $0<m\leq n$.\\
\end{flushleft}
\label{alg:lead}
\begin{algorithmic}[1]
\State{Choose the landmark set $\{(a_k, b_k)\}_{k=1}^m$;} 
\State{Build the $n\times m$ landmark affinity matrix $\bar{\mathbf{W}}^{(1)}$ by $\bar{\mathbf{W}}^{(1)}_{ij}=\tilde{K}^{(1)}_{\epsilon_1}(s_i, a_k)$;}
\State{Build the $n\times m$ landmark affinity matrix $\bar{\mathbf{W}}^{(2)}$ by $\bar{\mathbf{W}}^{(2)}_{ij}=\tilde{K}^{(2)}_{\epsilon_2}(r_i, b_k)$;}
\State{Construct $\bar{\mathbf{D}}^{(2)}=\texttt{diag}\big(\bar{\mathbf{W}}^{(2)\top}\bar{\mathbf{W}}^{(2)}\mathbf{1}_{m}\big)$;}
\State{Construct the landmark diffusion matrix $\bar{\mathbf{M}}^{(2)}_\alpha=\bar{\mathbf{W}}^{(2)}(\bar{\mathbf{D}}^{(2)})^{-\alpha}$;}
\State{Construct the $n\times n$ diagonal landmark degree matrices $\bar{\mathbf{D}}^{(1)}_\alpha$ by $\bar{\mathbf{D}}^{(1)}_{\alpha}=\texttt{diag}\big(\bar{\mathbf{W}}^{(1)}\bar{\mathbf{M}}^{(2)\top}_\alpha\mathbf{1}_{n}\big)$;}
\State{Construct the landmark diffusion matrix $\bar{\mathbf{M}}^{(1)}_\alpha=(\bar{\mathbf{D}}^{(1)}_\alpha)^{-1}\bar{\mathbf{W}}^{(1)}$;}
\State{Apply EVD on the matrix $\bar{\mathbf{M}}^{(2)\top}_\alpha\bar{\mathbf{M}}^{(1)}_\alpha=\mathbf{V}\bar{\mathbf{\Lambda}} \mathbf{V}^{-1}$ where eigenvalues are $\bar{\lambda}_1\geq \cdots\geq \bar{\lambda}_m$;}
\State{Set $\bar{\mathbf{U}}=\mathbf{M}^{(1)}_\alpha\mathbf{V}=[\bar{u}_1\cdots\bar{u}_{m}]$; \Comment{Approximate $\mathbf{U}$}}
\State{Set $\bar{\mathbf{U}}_q=[\bar{u}_2\cdots\bar{u}_{q+1}]$ and $\bar{\mathbf{\Lambda}}_q=[\bar{\lambda}_2\cdots\bar{\lambda}_{q+1}]$;}\end{algorithmic}
\textbf{Output:} $e^\top_i \bar{\mathbf{U}}_q\,\bar{\mathbf{\Lambda}}_q^t$ as the embedded point of $(s_i, r_i)$, where $i=1,\ldots,n$.
\end{algorithm}

The computational complexity of evaluating the landmark affinity matrices and landmark diffusion matrices, obtaining $\bar{\mathbf{M}}^{(2)\top}_\alpha\bar{\mathbf{M}}^{(1)}_\alpha$ and the EVD of $\bar{\mathbf{M}}^{(2)\top}_\alpha\bar{\mathbf{M}}^{(1)}_\alpha$ are $\Theta(nm)$, $O(nm^2)$ and $O(m^{2.81})$ respectively. Therefore, the overall computational complexity of $\alpha$-LAD is $\Theta(nm)+O(nm^2)$ since we assume $m<n$. Eventually, if $m<n^{1/2}$, $\alpha$-LAD has a preferable computational complexity compared with AD.

\subsection{Justification of LAD from a linear algebra perspective}\label{section justification of LAD on linear algebra level}

Suppose we have a vector $v \in \mathbb{R}_{\ge 0}^n$ of mass distribution on the measurements from the second sensor such that $v^{\top} \mathbf{1}_{n}=1$. 
Geometrically, we could understand AD as two diffusion steps of $v$. 
The first step is a diffusion on the second dataset, i.e., $\mathbf{M}^{(2)}v$. Then, the result is bijectively transferred to the first dataset by considering the range of $\mathbf{M}^{(2)}$ as the domain of $\mathbf{M}^{(1)}$). The second diffusion step is then applied on the first dataset, i.e., $\mathbf{M}^{(1)}\mathbf{M}^{(2)}v$). It is noteworthy that when $\mathcal{X}_1=\mathcal{X}_2$, AD reduces to a single diffusion step with double diffusion time. 

To speed up the computation time, $\alpha$-LAD attempts to approximate the effective diffusion process above by incorporating intermittent diffusion steps on the landmarks.
The initial diffusion is transferred from ${\mathcal{X}}_2$ to $\tilde{\mathcal{Z}}_2$ by considering $\bar{\mathbf{M}}_\alpha^{(2)}v$. Then, the result is bijectively transferred from $\tilde{\mathcal{Z}}_2$ to $\tilde{\mathcal{Z}}_1$. Finally, diffusion from $\tilde{\mathcal{Z}}_1$ to ${\mathcal{X}}_1$ is applied, i.e., $\bar{\mathbf{M}}_\alpha^{(1)}\bar{\mathbf{M}}_\alpha^{(2)}v$. Overall, we obtain a diffusion from ${\mathcal{X}}_2$ to ${\mathcal{X}}_1$ as in AD, but the intermittent diffusion steps are performed on the landmarks.

First, we show that the landmark AD matrix $\bar{\mathbf{M}}_\alpha$ is a diffusion matrix. By construction, based on the kernels, all the entries of $\bar{\mathbf{M}}_\alpha$ are nonnegative. By Equations (\ref{eq:diag1}) and (\ref{eq:markov1}), we have
\begin{align}
\bar{\mathbf{M}}^{(1)}_\alpha\bar{\mathbf{M}}^{(2)\top}_\alpha \mathbf{1}_n
&\,=\texttt{diag}\big(\bar{\mathbf{W}}^{(1)}\bar{\mathbf{M}}^{(2)\top}_\alpha\mathbf{1}_{n}\big)^{-1}\bar{\mathbf{W}}^{(1)}(\bar{\mathbf{D}}^{(2)})^{-\alpha}\bar{\mathbf{W}}^{(2)}\mathbf{1}_n\nonumber\\
&\,=\texttt{diag}\big(\bar{\mathbf{W}}^{(1)} (\bar{\mathbf{D}}^{(2)})^{-\alpha} \bar{\mathbf{W}}^{(2)} \mathbf{1}_{n}\big)^{-1}\bar{\mathbf{W}}^{(1)}(\bar{\mathbf{D}}^{(2)})^{-\alpha}\bar{\mathbf{W}}^{(2)}\mathbf{1}_n=\mathbf{1}_n\,,
\end{align}
which concludes that $\bar{\mathbf{M}}_\alpha$ is a diffusion matrix. 
Moreover, the above calculation shows that the $L^\infty$ norm of $\bar{\mathbf{M}}^{(1)}_\alpha\bar{\mathbf{M}}^{(2)\top}_\alpha$ is $1$.
Since the spectral norm is bounded by the $L^\infty$ norm, all eigenvalues, while might be complex, are bounded by $1$. When all the eigenvalues of  $\bar{\mathbf{M}}^{(1)}_\alpha\bar{\mathbf{M}}^{(2)\top}_\alpha$ are simple, the EVD exists. However, in general, the matrix might be defective. In Section \ref{sec:main_thm}, we will show that under the manifold setup, asymptotically, $\bar{\mathbf{M}}^{(1)}_\alpha\bar{\mathbf{M}}^{(2)\top}_\alpha$ pointwisely converges to a self-adjoint operator. By this result and numerical confirmation, we hypothesize that the eigenvalues of $\bar{\mathbf{M}}^{(1)}_\alpha\bar{\mathbf{M}}^{(2)\top}_\alpha$ are asymptotically real, and $\bar{\mathbf{M}}^{(1)}_\alpha\bar{\mathbf{M}}^{(2)\top}_\alpha$ is asymptotically non-defective. 

Next, we discuss the relationship between the spectral structures of $\bar{\mathbf{M}}^{(2)\top}_\alpha\bar{\mathbf{M}}^{(1)}_\alpha$ and $\bar{\mathbf{M}}^{(1)}_\alpha\bar{\mathbf{M}}^{(2)\top}_\alpha$.
Suppose $\bar{\mathbf{M}}^{(1)}_\alpha$ is of full rank. Let $\mathbf{V}\in Gl(m)$ be the eigenvector matrix of $\bar{\mathbf{M}}^{(2)\top}_\alpha \bar{\mathbf{M}}^{(1)}_\alpha$. By the rank-nullity theorem,  the null space of $\bar{\mathbf{M}}^{(1)}_\alpha$ is degenerate and hence $\bar{\mathbf{U}}:=\bar{\mathbf{M}}^{(1)}_\alpha \mathbf{V}$ is of full rank as well. By a direct calculation, the columns of $\bar{\mathbf{U}}$ form the $m$ right eigenvectors of $\bar{\mathbf{M}}^{(1)}_\alpha \bar{\mathbf{M}}^{(2)\top}_\alpha$. If the rank of $\bar{\mathbf{M}}^{(1)}_\alpha\bar{\mathbf{M}}^{(2)\top}_\alpha$ is $m$, then we obtain all the nontrivial right eigenvectors. 
Note that in general, even if the rank of $\bar{\mathbf{M}}^{(1)}_\alpha$ is less than $m$, we can still recover the nontrivial eigenvectors of $\bar{\mathbf{M}}^{(1)}_\alpha\bar{\mathbf{M}}^{(2)\top}_\alpha$ from the nontrivial eigenvectors of $\bar{\mathbf{M}}^{(2)\top}_\alpha\bar{\mathbf{M}}^{(1)}_\alpha$. Overall, we get that we can speed up the computation of the nontrivial eigenvectors of the $n \times n$ matrix $\bar{\mathbf{M}}^{(1)}_\alpha\bar{\mathbf{M}}^{(2)\top}_\alpha$ using $\alpha$-LAD by considering the eigendecomposition of the smaller $m \times m$ matrix $\bar{\mathbf{M}}^{(2)\top}_\alpha\bar{\mathbf{M}}^{(1)}_\alpha$.

We remark that the normalization factor $\alpha\in[0,1]$ plays a similar role as in diffusion maps \cite{coifman2006}. We will systematically study its effect in Section \ref{sec:main_thm} under the manifold setup. Broadly, we will show that under the manifold setup, when $\alpha=1$, the impact of the landmark distribution is minimized. Conversely, when $\alpha=0$, $0$-LAD is impacted by the landmark distribution. In the special case that $\alpha=0$ and $\mathcal{X}_1=\mathcal{X}_2$, $0$-LAD is reduced to ROSELAND \cite{shen2020,shen2022}.

\section{Asymptotic Analysis of LAD under the Manifold Setup}
\label{sec:main_thm}
To further understand the behavior of LAD and connect it with the original AD, thereby showing that it recovers the common manifold, we continue to study the asymptotic behavior of LAD under the manifold setup.

\subsection{Manifold Setting}

Suppose there is a common structure that two sensors aim to capture, represented by a $d$-dimensional compact and smooth manifold $\mathcal{M}$ without boundary. This manifold is referred to as the common manifold. In practice, different sensors may be of different modalities and capture different ``views'' or information about the common manifold. These differences are captured by different embeddings of $\mathcal{M}$. We mention that it is possible to consider a more complicated model that includes sensor-specific variables as in \cite{lederman2018learning}. Let $\mathcal{M}$ be smoothly embedded into $\mathbb{R}^{p_1}$ via $\iota^{(1)}$ and into $\mathbb{R}^{p_2}$ via $\iota^{(2)}$, where $p_1$ and $p_2$ may not be equal. Here, $\iota_{(\ell)}$ encodes the features of the $\ell$-th sensor, and $\mathbb{R}^{p_\ell}$ is the space that hosts the acquired information. The collected dataset is modeled by independently and identically distributed (i.i.d.) samples from two random vectors, $X^{(\ell)}:=\iota^{(\ell)}\circ X$, where $X:(\Omega,\mathcal{F},P)\rightarrow \mathcal{M}$ is a random vector with induced measure on $\mathcal{M}$, denoted by $\nu:=X_\ast P$, defined on the Borel sigma algebra on $\mathcal{M}$. It is important to note that we cannot directly access $\mathcal{M}$, but rather observe it through a dataset collected by two sensors, $\{(s_i, r_i)\}_{i=1}^n$ where $s_i = \iota^{(1)}(x_i) \in \mathbb{R}^p$ and $r_i = \iota^{(2)}(x_i) \in \mathbb{R}^p$. In contrast to many sensor fusion techniques that merely aim to merge the multi-sensor information, in LAD, as in AD, the ultimate goal is to recover the common manifold by recovering $\mathcal{X}=\{x_i\}_{i=1}^n\subset\M$ from the dataset $\{(s_i,r_i)\}_{i=1}^n$.

Since $\iota^{(\ell)}$ is a smooth embedding, $\iota^{(1)}(\mathcal{M})$ and $\iota^{(2)}(\mathcal{M})$ diffeomorphic to each other and $\Phi:=\iota^{(1)}\circ{\iota^{(2)}}^{-1}$ is a diffeomorphism. Assume the metric of $\iota^{(\ell)}(\mathcal{M})$, denoted as $g^{(\ell)}$, is induced from the canonical metric of $\mathbb{R}^p$. Let $dV^{(\ell)}$ be the measure associated with the Riemannian volume form induced from $g^{(\ell)}$. Assume $\nu_{\mathcal{M}}$ is absolutely continuous with respect to $dV^{(\ell)}$. By the Radon-Nikodym theorem, there exists a probability density function (p.d.f.) $p^{(\ell)}$ on $\mathcal{M}$ such that $d\nu=p^{(\ell)}dV^{(\ell)}$. 

The landmark set $\Z$ is modeled in a similar way. Assume $\mathcal{Z}=\{z_k\}_{k=1}^m$ is sample from a $\mathcal{M}$-valued random variable $Z:(\Omega,\mathcal{F},P)\rightarrow \mathcal{M}$ with induced measure $\nu^\Z=Z_*P$. Assume $Z$ is independent of $X$. 
By the Radon-Nikodym theorem, there exists a p.d.f. $p^{(\ell)}_\Z$ of $Z^{(\ell)}$ on $\mathcal{M}$ such that $d\nu^\Z=p^{(\ell)}_\Z dV^{(\ell)}$, where $\ell=1,2$.

With the sampled data $\{(s_i,r_i)\}_{i=1}^n\subset \mathbb{R}^{p_1}\times \mathbb{R}^{p_2}$ and the landmark set $\{(a_i,b_i)\}_{i=1}^m\subset \mathbb{R}^{p_1}\times \mathbb{R}^{p_2}$, where $a_i=\iota^{(1)}(z_i)$ and $b_i=\iota^{(2)}(z_i)$, we can run LAD. Below we provide a series of analyses studying the behavior of LAD under the manifold setup. Note that while in general we can consider embeddings into more general metric spaces, in this analysis we focus on $(\mathfrak{M}_\ell,d_\ell)=(\mathbb{R}^{p_\ell},\|\cdot\|)$ for $\ell=1,2$.

With the above model, to proceed with the analysis, we specify the following assumptions. 

\begin{assumption}\label{asp:assumption}

\begin{enumerate}
\item The $d$-dimensional manifold $\mathcal{M}$ is compact without boundary and embedded into $\mathbb{R}^{p_\ell}$ via $C^4(\mathcal{M})$ map $\iota^{(\ell)}$.
\item The kernel functions are $C^3$, positive and decay exponentially fast; that is, exist $c_1,c_2>0$ such that
\begin{equation*}
\tilde{K}^{(\ell)}(t)<c_1e^{-c_2t^2} \quad\text{and}\quad |\partial_t \tilde{K}^{(\ell)}(t)|<c_1e^{-c_2t^2}\,.
\end{equation*}
Furthermore, define $\mu_{r,k}^{(\ell)}:=\int_{\mathbb{R}^d}\|x\|^r \partial^{(k)} \tilde{K}^{(\ell)}(\|x\|) d x$. Assume $\mu^{(\ell)}_{0,0}=1$ for $\ell=1,2$.
\item $\nu$ is absolutely continuous with respect to $dV^{(\ell)}$. Moreover, $p^{(\ell)}\in C^{4}(\M)$ and is bounded below from zero.
\end{enumerate}
\end{assumption}

Without loss of generality, we analyze LAD assuming it starts from the second sensor. A similar analysis holds when it starts from the first sensor. First, we define the degree function on the second manifold. To simplify the tedious notation, we systematically use $K^{(\ell)}_\epsilon(x_i,x_j):=\tilde{K}^{(\ell)}\left(\frac{\|\iota^{(\ell)}(x_i)-\iota^{(\ell)}(x_j)\|^2}{\epsilon}\right)$ for $\ell=1,2$.
\begin{definition}\label{def:lan2kernel}
Define the {\em landmark kernel function} $K^{(2)}_{\lan,\epsilon}:\mathcal{M}\times\mathcal{M}\rightarrow\mathbb{R}_+$ as 
\begin{equation*}
K^{(2)}_{\lan,\epsilon}(z,z'):=\int_{\mathcal{M}}K_\epsilon^{(2)}(z, x)K_\epsilon^{(2)}(x,z')d \nu(x)
\end{equation*}
and define the {\em landmark degree function}, denoted as $d_{\lan,\epsilon}^{(2)}:\mathcal{M}\rightarrow\mathbb{R}_+$, by
\begin{equation*}
d_{\lan,\epsilon}^{(2)}(z)= \int_{\mathcal{M}} K^{(2)}_{\lan,\epsilon}(z,z') d \nu^\Z(z')\,.
\end{equation*}
\end{definition}

First, recall the following result from  \cite[Theorem 10]{shen2022}, which states that the quantity $d_{\lan,\epsilon}^{(2)}$ approximates the quantity $p^{(2)}p^{(2)}_\Z$. 

\begin{theorem}[Theorem 10 in \cite{shen2022}]
\label{thm:rose}
Take $f\in C^3(\mathcal{M})$. For $x\in\mathcal{M}$, we have
\begin{align*}
d_{\lan,\epsilon}^{(2)}(z)&\,=\int_\M K^{(2)}_{\mathrm{ref}, \epsilon}(z, z') p^{(2)}_{\Z}(z') dV^{(2)}(z')\\
&\,=\epsilon^{d}\bigg[  p^{(2)}(z) p^{(2)}_\Z(z)+\epsilon\frac{ \mu_{2,0}^{(2)}}{d}\bigg( p^{(2)}_\Z(z) \Delta p^{(2)}(z)+\frac{1}{2 } p^{(2)}(z) \Delta p^{(2)}_\Z(z)\\
& \qquad-w(z) p^{(2)}(z) p^{(2)}_\Z(z)+ \nabla p^{(2)}(z) \cdot \nabla p^{(2)}_\Z(z)\bigg)\bigg]+\mathcal{O}\left(\epsilon^{d+2}\right)\\
&\,:=\epsilon^dp^{(2)}(z) p^{(2)}_\Z(z)\left(1+\epsilon\bar{E}(z)\right)+\mathcal{O}\left(\epsilon^{d+2}\right)\,,
\end{align*}
where 
\[
\bar{E}(x):=\frac{\Delta p^{(2)}(x)}{p^{(2)}(x)}
+\frac{ \Delta p^{(2)}_\Z(x)}{2p^{(2)}_\Z(x)}
+ \frac{\nabla p^{(2)}(x) \cdot \nabla p^{(2)}_\Z(x)}{p^{(2)}(x) p^{(2)}_\Z(x)}-w(x)\,,
\]
$w(x)=\frac{1}{3}s(x)-\frac{d}{12|S^{d-1}|}\int_{S^{d-1}}[\Pi^{(2)}_x(\theta,\theta)]^2d\theta$, $s(x)$ is the scalar curvature at $x$, and $\Pi^{(2)}_x$ is the second fundamental form of the embedding at $x$, and $|S^{d-1}|$ is the volume of the canonical $(d-1)$-dim sphere.
\end{theorem}

Next, recall that that the landmark kernel induced by the landmark set depends on the normalization factor $\alpha$. We have the following definition.

\begin{definition}\label{def:lanADkernel}
($\alpha$-Landmark Alternating Diffusion Kernel) The $\alpha$-{\em landmark alternating diffusion kernel}, $K_{\lan,\epsilon,\alpha}:\mathcal{M}\times\mathcal{M}\rightarrow \mathbb{R}_+$, is defined by
\begin{equation*}
K_{\lan,\epsilon,\alpha}\left(x, y\right):=\int_{\mathcal{M}} \frac{K_\epsilon^{(1)}\left(x, z^{\prime}\right) K_\epsilon^{(2)}\left(z^{\prime}, y\right)}{[d_{\lan,\epsilon}^{(2)}(z')]^\alpha} d\nu^\Z\left(z^{\prime}\right) \,.
\end{equation*}
The $\alpha$-{\em landmark alternating degree function}, $d_{\lan,\epsilon,\alpha}:\mathcal{M}\rightarrow\mathbb{R}_+$, is defined by
\begin{equation*}
d_{\lan,\epsilon,\alpha}(x)= \int_{\mathcal{M}} K_{\lan,\epsilon,\alpha}(x,y) d \nu(y)\,.
\end{equation*}
\end{definition}

With the $\alpha$-landmark alternating diffusion kernel and $\alpha$-landmark alternating degree function, we can define the landmark alternating diffusion operator. 
\begin{definition}
\label{def:leadOP}
(Landmark Alternating Diffusion Operator) The landmark alternating diffusion operator, $T_{\lan,\epsilon}:C(\mathcal{M})\to C(\mathcal{M})$, is defined as
\begin{equation*}
T_{\lan,\epsilon,\alpha}f(x)=\int_{\mathcal{M}} \frac{K_{\lan,\epsilon,\alpha}(x,y)}{d_{\lan,\epsilon,\alpha}(x)}f(y)d\nu(y)\,.
\end{equation*}
\end{definition}
It is clear that $T_{\lan,\epsilon,\alpha}1=1$, where $1$ is a constant function with value $1$. 

It is immediate to see that the above quantities are continuous companions of landmark AD matrices used in the LAD algorithm. Denote
$\nu_n=\frac{1}{n}\sum_{i=1}^n\delta_{x_i}$
to be the empirical measure associated with the samples on the common manifold and
$\nu_m^\Z=\frac{1}{m}\sum_{i=1}^m\delta_{z_i}$ to be the empirical measure associated with the landmark set.
Note that $d^{(2)}_{\lan,\epsilon}(z_i)= \int_{\mathcal{M}} K^{(2)}_{\lan,\epsilon}(z_i,z') d \nu^\Z(z')$ is the continuous version of $d^{(2)}_{\lan,\epsilon,m}(z_i):=\int_{\mathcal{M}} K^{(2)}_{\lan,\epsilon,n}(z_i,z') d \nu_m^\Z(z')$, where $K^{(2)}_{\lan,\epsilon,n}(x_i,z'):=\int_{\mathcal{M}}K_\epsilon^{(2)}(z, x)K_\epsilon^{(2)}(x,z')d \nu_n(x)$.
By a direct expansion, we have $\bar{\mathbf{D}}^{(2)}_{ii}=\int_{\mathcal{M}} K^{(2)}_{\lan,\epsilon,n}(x_i,z') d \nu_m^\Z(z')$. 
Similarly, 
$d_{\lan,\epsilon,\alpha}(x_i)= \int_{\mathcal{M}} K_{\lan,\epsilon,\alpha}(x_i,y) d \nu(y)$ is the continuous version of $\int_{\mathcal{M}} K_{\lan,\epsilon,\alpha,m}(x_i,y) d \nu_n(y)$, where $K_{\lan,\epsilon,\alpha,m}(x, y):=\int_{\mathcal{M}} \frac{K_\epsilon^{(1)}\left(x, z^{\prime}\right) K_\epsilon^{(2)}\left(z^{\prime}, y\right)}{[d_{\lan,\epsilon,m}^{(2)}(z')]^\alpha} d\nu_m^\Z\left(z^{\prime}\right)$. By a direct expansion, we have
$\bar{\mathbf{D}}^{(1)}_{\alpha, ii}=\int_{\mathcal{M}} K_{\lan,\epsilon,\alpha,n}(x_i,y) d \nu_n(y)$. 
Along similar lines, we see that 
$K_{\lan,\epsilon,\alpha}(x_i,x_j)$ is the continuous version of
$\big(\bar{\mathbf{W}}^{(1)}\bar{\mathbf{M}}_\alpha^{(2)\top}\big)_{ij}$ 
and
$T_{\lan,\epsilon,\alpha}$ is the continuous version of the landmark AD matrix
$\bar{\mathbf{M}}_\alpha=\bar{\mathbf{M}}_\alpha^{(1)}\bar{\mathbf{M}}^{(2)\top}_\alpha$.

The main theoretical result is about the asymptotic behavior of $\bar{\mathbf{M}}_\alpha$, which is composed of the variance analysis and bias analysis. In short, in the variance analysis, asymptotically when $n\to \infty$, we quantify the stochastic fluctuation of the convergence of $\bar{\mathbf{M}}_\alpha$ to the integral operator $T_{\lan,\epsilon}$ caused by finite sampling points. In the bias analysis, asymptotically when $\epsilon\to 0$, we show that $T_{\lan,\epsilon}$ converges to a deformed Laplacian-Beltrami operator defined on $(\M,g^{(2)})$ and quantify the convergence rate. The proof of Theorem \ref{thm:main} is postponed to Section \ref{sec:proof}.

\begin{theorem}\label{thm:main}
Fix $x \in \mathcal{M}$. Fix normal coordinates around $x$ associated with $g^{(\ell)}$ and the associated exponential maps $\exp^{(\ell)}$, where $\ell=1,2$. Let $\mathcal{B}:=\{E_i\}_{i=1}^d \subset T_x \mathcal{M}$ be an orthonormal basis associated with $g^{(2)}$. Set 
\[
R_x=\Big[\mathrm{d} \exp^{(1)}\big|_0\Big]^{-1}\big[\mathrm{d}\iota^{(1)}\big]^{-1} \nabla \Phi\big[\mathrm{d} \iota^{(2)}\big]\Big[\mathrm{d}\exp^{(2)}\big|_0\Big]\in \mathbb{R}^{d\times d}\,,
\]
with all quantities represented with the basis $\mathcal{B}$. By the singular value decomposition (SVD), we have $R_x=U_x \Lambda_x V_x^T$, where $\Lambda_x=\operatorname{diag}\left[\lambda_1, \ldots, \lambda_d\right]$ with $\lambda_1\geq \lambda_2\geq\ldots\geq \lambda_d$, and $U_x,V_x\in O(d)$. Then, for a $f \in C^3(\mathcal{M})$, when $\epsilon>0$ is sufficiently small, we have
\begin{align*}
T_{\lan,\epsilon,\alpha}f(x)=f(x)\,&+\frac{\epsilon \mu_{2,0}^{(2)}}{2d}\Delta^{(2)}f(x)+\frac{\epsilon \mu_{2,0}^{(1)}}{2d}\sum^{d}_{i=1}\lambda_i\nabla^{(2)^2}_{E_iE_i}f(x)\\
&+\frac{\epsilon \mu_{2,0}^{(1)}}{d}\sum_{i=1}^d \lambda_i\left(\frac{\nabla^{(2)}_{E_i} p^{(2)}(x)}{p^{(2)}}+\frac{\nabla^{(2)}_{E_i} q_\alpha(x)}{q_\alpha(x)}\right)\nabla^{(2)}_{E_i} f(x)\\
&+\frac{\epsilon\mu_{2,0}^{(2)}}{d}\frac{\nabla^{(2)} p^{(2)}(x)\cdot\nabla^{(2)} f(x)}{p^{(2)}(x)}+\mathcal{O}(\epsilon^{3/2})\,.
\end{align*}
where $\nabla^{(\ell)}$ is the covariant derivative associated with $\iota^{(\ell)}$, $q_\alpha(x):=\frac{p_\Z^{(2)}(x)^{1-\alpha}}{p^{(2)}(x)^\alpha}$ and the implied constant depends on the scalar curvature at $x$, second fundamental form at $x$, $C^3$ norm of $f$ and $C^2$ norms of $p^{(2)}$ and $p^{(2)}_\Z$.
\end{theorem}

By a direct calculation, asymptotically when $\epsilon\to 0$ the operator $\frac{2d}{\mu_{2,0}^{(2)}}\frac{T_{\lan,\epsilon,\alpha}-I}{\epsilon}$ converges to
\begin{align*}
\mathcal{L}_{\alpha}:=&\Delta^{(2)}+2\frac{\nabla^{(2)} p^{(2)} }{p^{(2)}}\cdot\nabla^{(2)}+\frac{\mu_{2,0}^{(1)}}{\mu_{2,0}^{(2)}}\left\{\sum^{d}_{i=1}\lambda_i\left[\nabla^{(2)^2}_{E_iE_i}+2\left(\frac{\nabla^{(2)}_{E_i} p^{(2)}}{p^{(2)}}+\frac{\nabla^{(2)}_{E_i} q_\alpha}{q_\alpha}\right)\cdot\nabla^{(2)}_{E_i}\right]\right\}\,,
\end{align*}
which is self-adjoint in nature, and involves not only the intrinsic operator, the Laplace-Beltrami operator $\Delta^{(2)}$, but also first-order terms that depend on the sampling densities and landmark densities.
The theorem immediately leads to the following corollaries, which demonstrate the connection between LAD and other related algorithms.

\begin{corollary}
Consider $p^{(2)}=p^{(2)}_\Z$. If we take $\alpha=1/2$, then $q_\alpha=1$ and Theorem \ref{thm:main} is reduced to original AD; that is, for $f\in C^3(\mathcal{M})$, we have
\begin{align*}
T_{\lan,\epsilon,\alpha}f(x)=f(x)&+\frac{\epsilon \mu_{2,0}^{(2)}}{2d}\Delta^{(2)}f(x)+\frac{\epsilon \mu_{2,0}^{(1)}}{2d}\sum^{d}_{i=1}\lambda_i\nabla_{E_iE_i}^{(2)^2}f(x)\\
&+\frac{\epsilon \mu_{2,0}^{(1)}}{d}\left( \sum_{i=1}^d \lambda_i\frac{\nabla^{(2)}_{E_i} p^{(2)}(x)\nabla^{(2)}_{E_i} f(x)}{p^{(2)}}\right)\\
&+\frac{\epsilon \mu_{2,0}^{(2)}}{2d}\frac{\nabla^{(2)} p^{(2)}(x)\cdot\nabla^{(2)} f(x)}{p^{(2)}(x)}+\mathcal{O}(\epsilon^{3/2})\,.
\end{align*}
If we further assume $p^{(2)}$ is constant; that is, the sampling is uniform, then we obtain
\begin{align*}
T_{\lan,\epsilon,\alpha}f(x)=f(x)+\frac{\epsilon \mu_{2,0}^{(2)}}{2d}\Delta^{(2)}f(x)+\frac{\epsilon \mu_{2,0}^{(1)}}{2d}\sum^{d}_{i=1}\lambda_i\nabla_{E_iE_i}^{(2)^2}f(x)+\mathcal{O}(\epsilon^{3/2})\,.
\end{align*}
\end{corollary}

This corollary implies that if the landmark distribution is the same as the data sample distribution; that is, $p^{(2)}=p^{(2)}_\Z$, then if we take $\alpha=1/2$, the asymptotic operator is the same as that of AD. 

\begin{corollary}\label{Corollary alpha=1 case}
Consider $\alpha=1$. Then $q_\alpha(x)=\frac{1}{p^{(2)}(x)}$, which is independent on $p_\Z$. For $f\in C^3(\mathcal{M})$, we have
\begin{align*}
T_{\lan,\epsilon,\alpha}f(x)=f(x)&\,+\frac{\epsilon \mu_{2,0}^{(2)}}{2d}\Delta^{(2)}f(x)+\frac{\epsilon \mu_{2,0}^{(1)}}{2d}\sum^{d}_{i=1}\lambda_i\nabla_{E_iE_i}^{(2)^2}f(x)\\
&+\frac{\epsilon \mu_{2,0}^{(2)}}{d}\frac{\nabla^{(2)} p^{(2)}(x)\cdot\nabla^{(2)} f(x)}{p^{(2)}(x)}+\mathcal{O}(\epsilon^{3/2})\,.
\end{align*}
If we further assume $p^{(2)}$ is constant; that is, the sampling is uniform, then we obtain
\begin{align*}
T_{\lan,\epsilon,\alpha}f(x)=f(x)+\frac{\epsilon \mu_{2,0}^{(2)}}{2d}\Delta^{(2)}f(x)+\frac{\epsilon \mu_{2,0}^{(1)}}{2d}\sum^{d}_{i=1}\lambda_i\nabla_{E_iE_i}^{(2)^2}f(x)+\mathcal{O}(\epsilon^{3/2})\,.
\end{align*}
\end{corollary}

This corollary implies that the normalization factor $\alpha$ plays the same role as the $\alpha$-normalization in the original diffusion map algorithm \cite{coifman2006} in the sense that the landmark distribution does not impact the asymptotic operator. 
Clearly, while the asymptotic operator does not depend on the landmark set, the operator is in general different from that of AD, unless $p^{(2)}$ is constant.

\begin{corollary} 
Consider $p_\Z^{(2)}\propto \frac{1}{{p^{(2)}}^3}$ and take $\alpha=1/2$, $K^{(1)}=K^{(2)}$. We have $q_\alpha(x)=\frac{1}{p^{(2)}(x)^2}$. For $f\in C^3(\mathcal{M})$, we have
\begin{align*}
T_{\lan,\epsilon,\alpha}f(x)=f(x)&+\frac{\epsilon \mu_{2,0}^{(2)}}{d}\left( \sum_{i=1}^d (1-\lambda_i)\frac{\nabla^{(2)}_{E_i} p^{(2)}(x)\nabla^{(2)}_{E_i} f(x)}{p^{(2)}}\right)\\
&+\frac{\epsilon \mu_{2,0}^{(2)}}{2d}\bigg[\Delta^{(2)}f(x)+\sum^{d}_{i=1}\lambda_i\nabla_{E_iE_i}^{(2)^2}f(x)\bigg]+\mathcal{O}(\epsilon^{3/2})\,.
\end{align*}
Furthermore, if $R_x=I_d$, then $\lambda_i=1$ and the result is independent on probability density. If $\iota^{(1)}=\iota^{(2)}$, then we obtain a landmark design result to recover the Laplace-Beltrami operator that is similar to Remark 1 in \cite{shen2020}.
\end{corollary}

Next, we present the variance analysis of LAD. The proof is postponed to Section \ref{sec:proof}. We show that as $n \rightarrow \infty$, the matrix $I_n-\left(\mathbf{D}^{(1)}_\alpha\right)^{-1} \mathbf{W}^{(1)}_\alpha\mathbf{M}^{(2)}_\alpha$ behaves asymptotically like the integral operator $1-T_{\lan,\epsilon,\alpha}$. The analysis primarily focuses on the stochastic fluctuation of the convergence from finite sampling points. Since the quantity of interest is of the order $\epsilon$, the stochastic fluctuation induced by the finite sampling points should be much smaller than $\epsilon$ so that we can obtain the quantify of interest.

\begin{theorem}
Take $\mathcal{X}=\left\{x_i\right\}_{i=1}^n$ and $\mathcal{Z}=\left\{z_k\right\}_{i=1}^m$, where $m=$ $\lceil n^\beta\rceil$ for some $0<\beta \leq 1$ and $\lceil n^\beta\rceil$ is the smallest integer greater than or equal to $x\in\mathbb{R}$. Take $f \in C^3\left(\M\right)$ and denote $\boldsymbol{f} \in \mathbb{R}^n$ such that $\boldsymbol{f}_i=f\left(x_i\right)$. Let $\epsilon=\epsilon(n)$ so that $\frac{\sqrt{\log n}}{n^{\beta / 2} \epsilon^{d/4+1}} \rightarrow 0$ and $\epsilon \rightarrow 0$ when $n \rightarrow \infty$. Then with probability higher than $1-\mathcal{O}\left(1 / n^2\right)$, we have
\begin{equation*}
\frac{1}{\epsilon}\left[\left(I_n-\left(\mathbf{D}^{(1)}_\alpha\right)^{-1} \mathbf{W}^{(1)}_\alpha\mathbf{M}^{(2)}_\alpha\right) \boldsymbol{f}\right](i)=\frac{f\left(x_i\right)-T_{\lan,\epsilon,\alpha}f(x_i)}{\epsilon} +\mathcal{O}\left(\frac{\sqrt{\log n}}{n^{\beta / 2} \epsilon^{d/4+1}}\right)
\end{equation*}
for all $i=1,2, \ldots, n$.
\label{thm:var}
\end{theorem}

We shall mention that the convergence rate is not optimal. In short, according to a numerical exploration reported in the end of Section \ref{sec:proof}, the convergence rate should depend on $\epsilon^{d/4+1/2}$ but not $\epsilon^{d/4+1}$. This discrepancy comes from the fact that we handle the convergence of landmark and dataset separately since handling the dependence caused by the diffusion among dataset and landmark is challenging. We leave this problem to our future work.

\section{Simulation Experiment Results}
\label{sec:simu}
In this section, we use simulated data to validate the theoretical behavior of LAD, and compare it with the original AD. All experiments were conducted on a computer equipped with an 8-core Apple M1 CPU, utilizing MATLAB software version 9.10 (R2021a) in its 64-bit configuration for experimental design and analysis.

In addition to comparing the computational time, we also evaluate how accurately $\alpha$-LAD approximates AD. We consider three quantities. Denote $(\lambda_l,u_l)$ as the $l$-th eigenpair of $\alpha$-LAD, and $(\mu_l,v_l)$ as the $l$-th eigenpair of AD. The $l$-th eigenvalue difference is quantified by $|\lambda_l-\mu_l|/\mu_l$, while the eigenvector difference is quantified by the inner product of $\langle u_l,v_l\rangle/\|u_l\|\|v_l\|$. Third, we quantify the embedding similarity. Assume we have $n$ points and two embeddings $E_1$ and $E_2$, where the embedding dimension is $q$. The third quantity is 
\[
d(E_1,\,E_2):=\min_{O\in O(q)}\Big(\frac{1}{n}\sum^n_{i=1}\|E_1(i)-OE_1(i)\|_2\Big); 
\]
that is, the embedded points are rotationally aligned before evaluating the mean pairwise distance. In this section, unless otherwise mentioned, the embedding similarity is evaluated with $q=3$.

\subsection{The impact of landmark size on $\alpha$-LAD}
We study the computational time and how well LAD approximates AD with various landmark sizes
Consider a torus $\mathbb{T}^2$ and a deformed torus $S$, defined as 
\begin{equation*}
\mathbb{T}^2=\left\{(x,y,z)\left\vert
\begin{aligned}
& x(u, v)=(R+r \cos v) \cos u \\
& y(u, v)=(R+r \cos v) \sin u \\
& z(u, v)=r \sin v
\end{aligned}
\right.\right\}
\end{equation*}
and
\begin{equation*}
S=\left\{(x,y,z)\left\vert
\begin{aligned}
    & x(u, v)=\left[R+(1+0.5\cos4u)r \cos v\right] \cos u \\
    & y(u, v)=\left[R+(1+0.5\cos4u)r \cos v\right] \sin u \\
    & z(u, v)=(1+0.5\cos4u)r \sin v
\end{aligned}
\right.\right\}\,,
\end{equation*}
where $R = 2.5$, $r = 1$, and $u, v \in [0, 2\pi)$. The factor $1 + 0.5 \cos 4u$ deformed the tube radius of the torus. There exist a diffeomorphism map, $\Phi_2: \mathbb{T}^2 \to S$. We uniformly sample 5000 points on $\mathbb{T}^2$ and map them to $S$ using $\Phi_2$, and collect these points as $\{(r_i, s_i)\}$. 
We consider different landmark sizes as $m=\lfloor 2^{i/2}\rfloor$ for $i=11,12,13,\ldots,21$, where $\lfloor\cdot \rfloor$ is the floor function, and generate landmarks independently for 30 times. We focus on $\alpha=0.5$ since $1/2$-LAD gives us the best approximation of AD according to the established theorem. 
The results are shown in Figure \ref{fig:tor_time_ev}. Clearly, when the landmark size increases, the computational time increases, but the top eigenvalues and eigenvectors of $1/2$-LAD better approximate those of AD. 

Furthermore, if we increase $n$ to $500,000$ ($1,000,000$ respectively), a typical laptop would no longer be able to compute AD. However, when using LAD with $m=1,000$, the computation time is $4.5$ ($12.3$ respectively) minutes. The main limitation going beyond one million points is out of the memory of the laptop.

\begin{figure}[ht]
\begin{center}
\includegraphics[width=1.\textwidth]{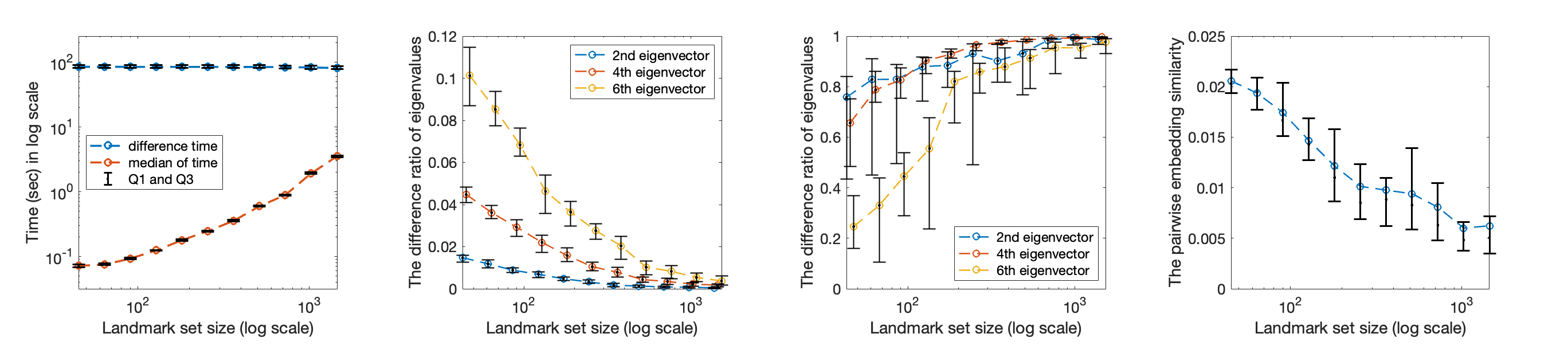}
\end{center}
\caption{From left to right: a comparison of computational time between AD and $1/2$-LAD with different landmark sizes, the difference ratio of the 1st, 3rd, 5th non-trivial eigenvalues between AD and $0.5$-LAD, the inner product between the 1st, 3rd, 5th non-trivial eigenvectors of AD and $0.5$-LAD, where we plot top 6 non-trivial eigenvectors, and the embedding similarity between AD and $0.5$-LAD.}
\label{fig:tor_time_ev}
\end{figure}

\subsection{The dependence on the initial diffusion} 
Consider the canonical $\mathbb{S}^1$ and Trefoil knot $C$, defined as
\begin{align*}
\mathbb{S}^1 &= \{(\cos \theta, \, \sin \theta)\} \subset\mathbb{R}^2, \\
C &= \{(\sin \theta + 2 \sin 2 \theta,\, \cos \theta - 2 \cos 2 \theta, \, - \sin 3 \theta)\}\subset\mathbb{R}^3\,,
\end{align*}
where $\theta \in [-\pi, \pi)$.
Note that $\mathbb{S}^1$ and $C$ are diffeomorphic via the map $\Phi_0 : \mathbb{S}^1 \to C$:
\begin{equation*}
\Phi_0(x, y) = \left\{
\begin{aligned}
x &= \sin \left( \arctan \frac{y}{x} \right) + 2 \sin 2 \left( \arctan \frac{y}{x} \right) \\
y &= \cos \left( \arctan \frac{y}{x} \right) - 2 \cos 2 \left( \arctan \frac{y}{x} \right) \\
z &= -\sin 3 \left( \arctan \frac{y}{x} \right)\,.
\end{aligned}
\right.
\end{equation*}
Uniformly sampling 1000 points from $\mathbb{S}^1$ and mapping them to $C$ using $\Phi_0$, we collect these points as the dataset $\tilde{\mathcal{X}}$. We then uniformly choose 100 landmark points from $\tilde{\mathcal{X}}$ and set $\alpha=0.5$. In Figure \ref{fig:knot_4}, we observe the impact of the initial sensor chosen for diffusion in both AD and $0.5$-LAD. Notably, the results of both $0.5$-LAD and AD are contingent upon the selected initial sensor. Additionally, the resulting embedding closely resembles the diffusion maps of the ending sensor, indicating that the embedding is predominantly influenced by diffusion in the ending sensor, which encodes the geometric information for the embedding. While $C$ is diffeomorphic to $\mathbb{S}^1$, it possesses a distinctive exterior topology different from $\mathbb{S}^1$. However, since AD and LAD are both diffusion-based algorithms capturing local geometric information, they ultimately cannot capture this exterior topology.

\begin{figure}[ht]
\begin{center}
\includegraphics[width=0.95\textwidth]{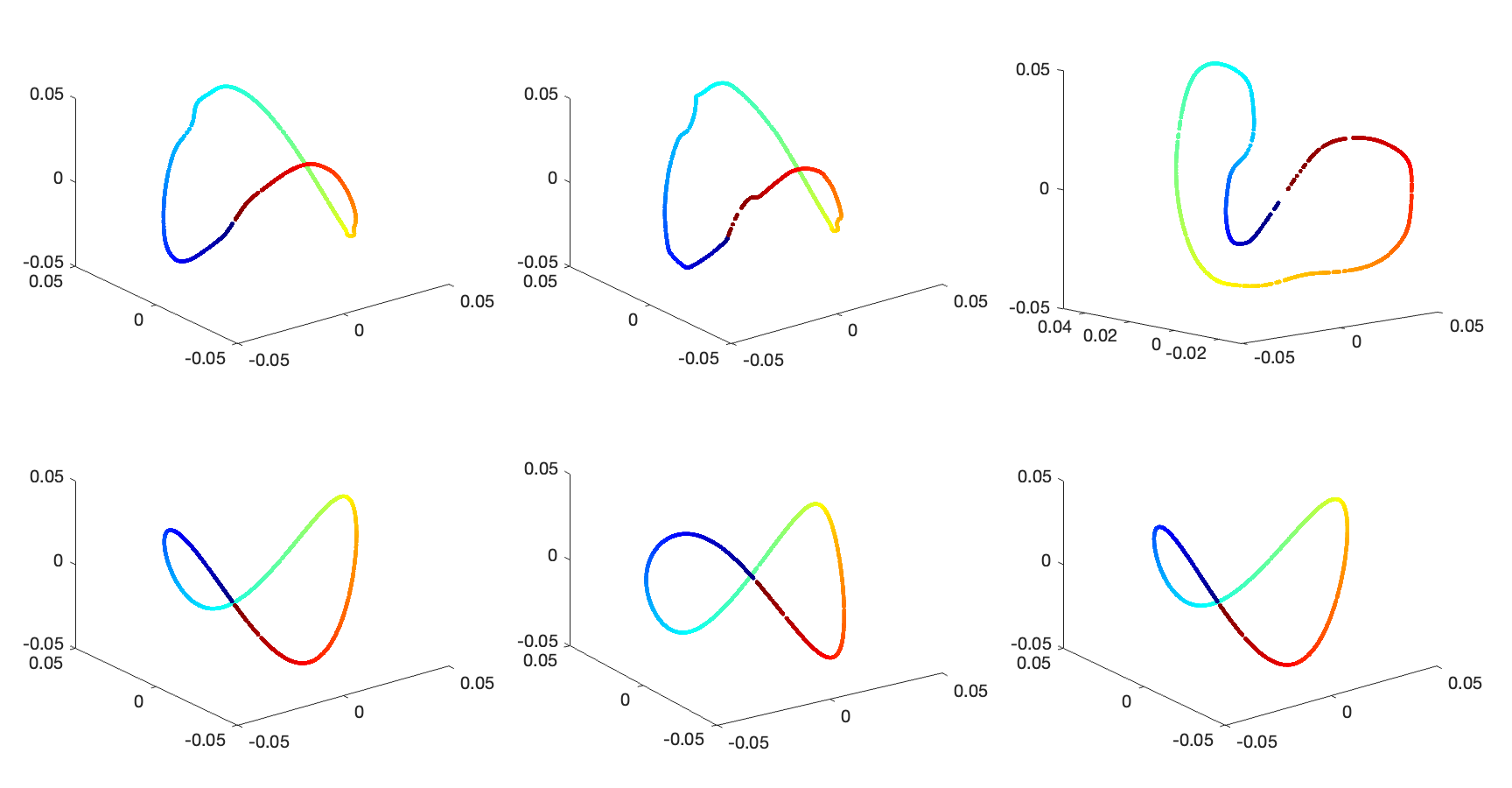}
\end{center}
\caption{Upper left: AD starting from $\mathbb{S}^1$. Upper middle: $0.5$-LAD starting from $\mathbb{S}^1$. Upper Right: DM of dataset on $C$. Lower left: AD starting from Trefoil knot $C$. Lower middle: $0.5$-LAD starting from Trefoil knot $C$. Lower right: DM of dataset on $\mathbb{S}^1$.}
\label{fig:knot_4}
\end{figure}

\subsection{The impact of landmark distribution and $\alpha$-normalization}\label{sec:uni_alpha1}

To study the impact of landmark distribution, consider the canonical $\mathbb{S}^1$ and ellipse $E$, defined as
\begin{align*}
\mathbb{S}^1 &= \{(\cos \theta, \, \sin \theta)\} \subset\mathbb{R}^2, \\
E &= \{(2\cos \theta, \, \sin \theta)\} \subset\mathbb{R}^2\,,
\end{align*}
where $\theta \in [-\pi, \pi)$. We assume $p^{(2)}$ is constant and define $\Phi_1:\mathbb{S}^1\rightarrow E$ by
\begin{equation*}
\Phi_1(x,y)=\left(2\cos\left(\arctan\frac{y}{x}\right), \sin\left(\arctan\frac{y}{x}\right)\right)\,.
\end{equation*}
The dataset is generated by uniformly sampling $n$ points on $\mathbb{S}^1$  and mapping them to $E$ by $\Phi_1$. We consider the following two scenarios. In the first scenario, $p_\Z^{(2)}$ is uniform, and in the second scenario, we set the non-uniform $p_\Z^{(2)}(\theta)=\frac{58}{50}[0.48\cos\theta+0.52]$. In both cases, we randomly sample 2500 points $\{(s_i, r_i):s_i\in E, r_i\in\mathbb{S}^1\}$ and choose 1000 landmark points according to $p_\Z^{(2)}$.

$\alpha$-LAD with various $\alpha$ under different $p_\Z^{(2)}$ setups are shown in the left column of Figure \ref{fig:eval_alpha}. We observe that when landmarks are uniformly selected, the choice of $\alpha$ seems to have less impact on the embeddings. However, in the case of non-uniform landmark selection, the embeddings vary and depend on the value of $\alpha$. When $\alpha=1$, the embedding (colored in cyan) becomes similar to LAD with uniform sampling. However, we will see below that when the landmark sampling is nonuniform, LAD fails to recover AD when $\alpha=1$. The quantification of the deviation of LAD from AD is shown in Figure \ref{fig:eval_alpha}. As $\alpha$ approaches 1, the impact of nonuniform landmark distribution on LAD is mitigated, corroborating the theoretical analysis.

\begin{figure}[ht]
\begin{center}
\includegraphics[trim=100 0 100 0, width=1.0\textwidth]{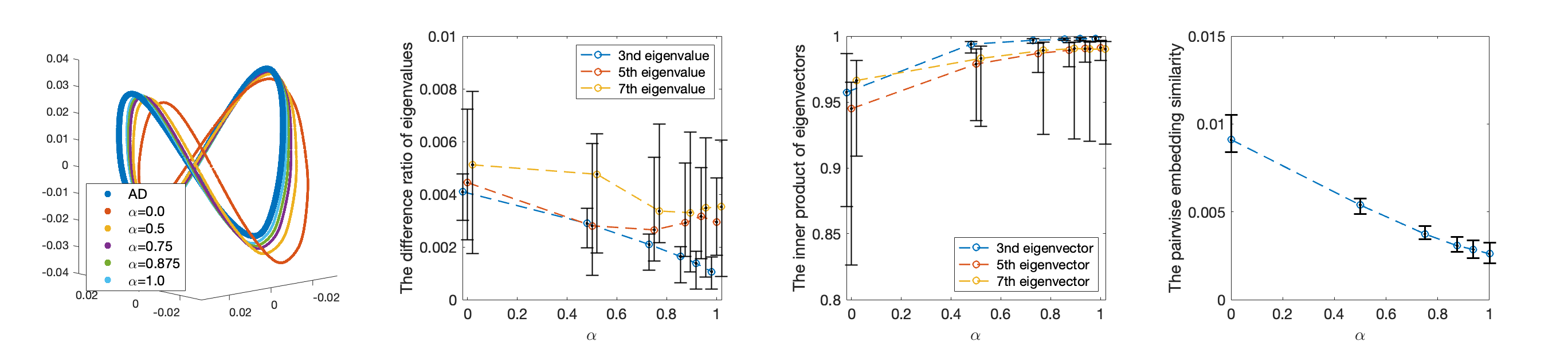} 
\includegraphics[trim=100 0 100 0, width=1.0\textwidth]{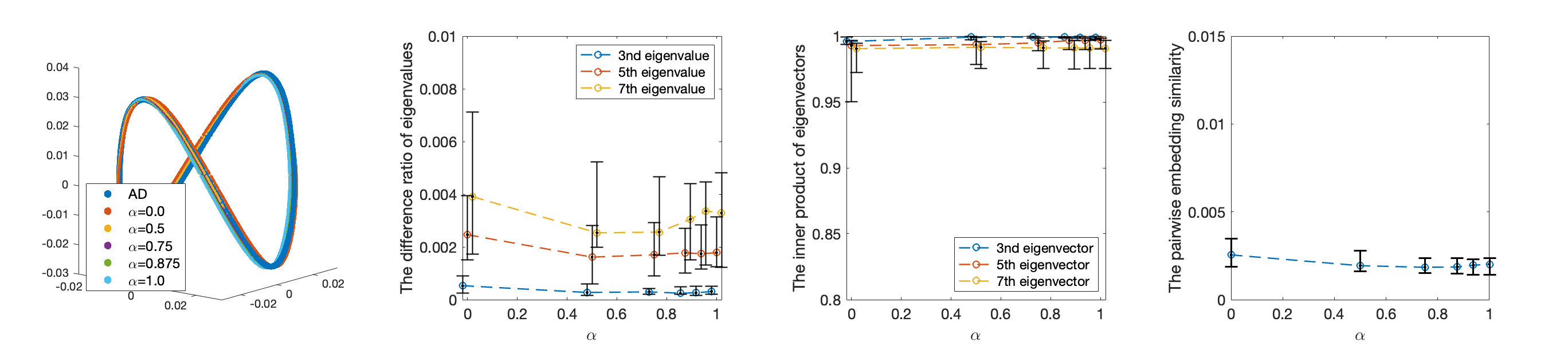}
\end{center}
\caption{Top row, from left to right: $\alpha$-LAD with various $\alpha$, the difference ratio of the 2nd, 4th, 6th non-trivial eigenvalues between AD and $\alpha$-LAD, the inner product between the 2nd, 4th, 6th non-trivial eigenvectors of AD and $\alpha$-LAD, where we plot the top 6 eigenvectors, and the embedding similarity between AD and $\alpha$-LAD, where the landmark distribution is non-uniform. Bottom row shows the same thing but when the landmark distribution is uniform. }
\label{fig:eval_alpha}
\end{figure}

We known from the theoretical analysis shown in Corollary \ref{Corollary alpha=1 case} that when $\alpha=1$, although the impact of nonuniform landmark sample on LAD is eliminated, we do not recover AD. To further explore this result, consider the same $\mathbb{S}^1$ setup. But this time, the dataset is sampled non-uniformly following the distribution $p^{(2)}$. Additionally, we arbitrarily select four different landmark distributions. See the top row of Figure \ref{fig:evo_alpha_to_1} for an illustration. We sample 3500 points $\{(s_i,r_i):s_i\in E, r_i\in\mathbb{S}^1\}$ following $p^{(2)}$ and choose 1000 landmark points following different landmark distributions. We label these four different $p^{(2)}_\Z$ cases as ``Case 1'' to ``Case 4.''

The embedding similarity under these different landmark distributions as $\alpha$ approaches 1 are shown in the top row of Figure \ref{fig:evo_alpha_to_1}. We see that as $\alpha$ approaches 1, the embeddings by LAD are more similar, but these embeddings are not necessarily similar to the embedding by AD. The quantification of embedding similarities is shown at the bottom row of Figure \ref{fig:evo_alpha_to_1}.

\begin{figure}[ht]
\begin{center}
\includegraphics[trim=100 0 100 0, width=1.\textwidth]{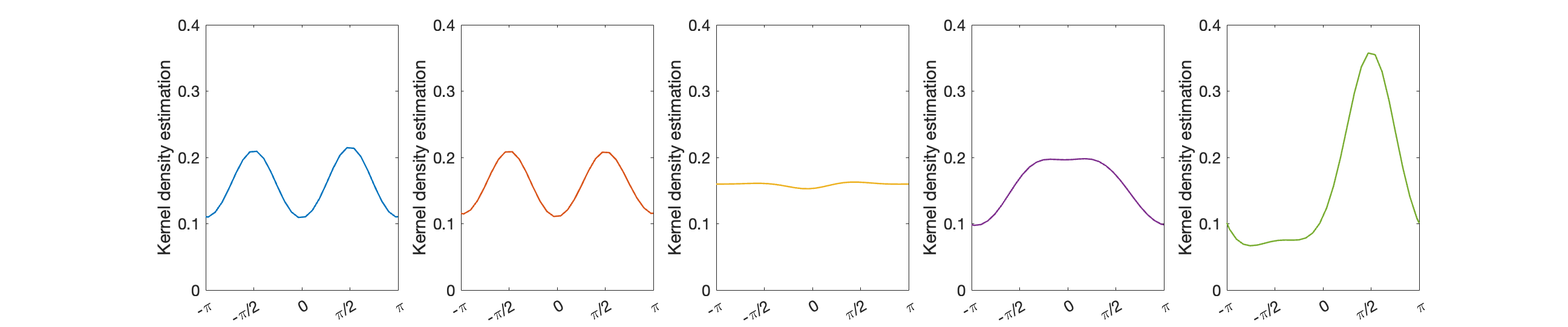}
\includegraphics[trim=100 0 100 0, width=1.\textwidth]{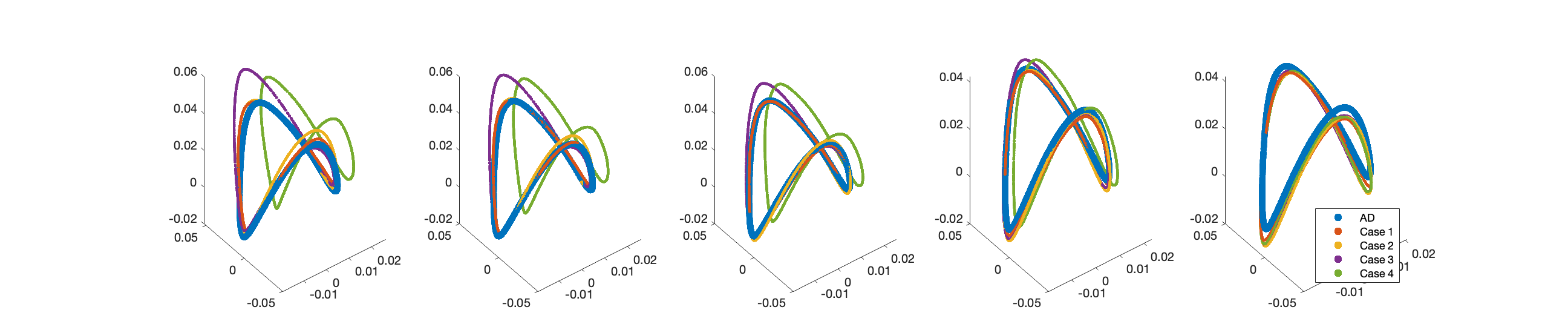}
\includegraphics[trim=100 0 100 0, width=1.\textwidth]{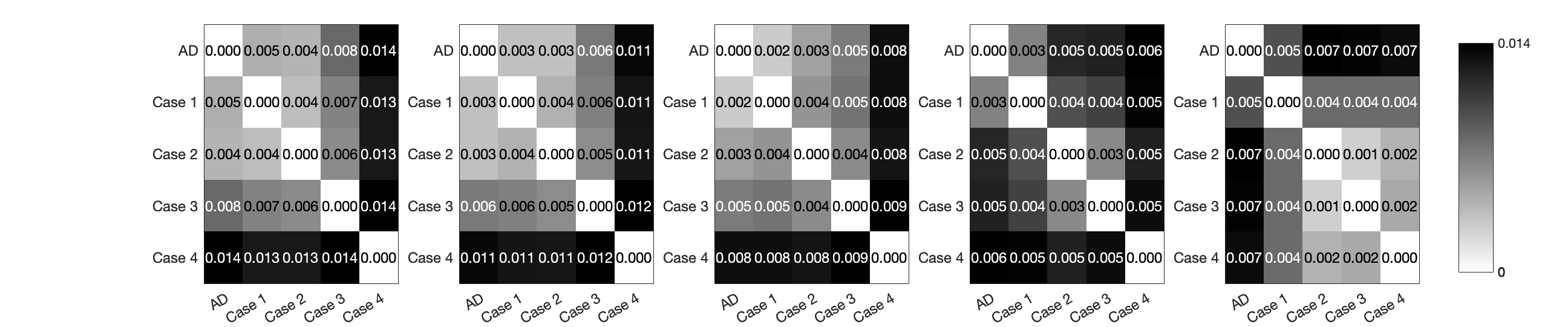}
\end{center}
\caption{Illustration of various embeddings by $\alpha$-LAD under different landmark distributions. Top row, from left to right: the sampling density function $p^{(2)}$ and four landmark distributions, labeled ``Case 1'' to ``Case 4.'' 
Middle row, from left to right: the LAD with $\alpha=0, 1/4, 1/2, 3/4$, and $1$, respectively. The blue curve is the result of AD, which is the same from left to right. The curves in other colors correspond to different landmark distributions. Lower row, from left to right: the pairwise embedding similarity between different embeddings with $\alpha=0, 1/4, 1/2, 3/4$, and $1$, respectively.}
\label{fig:evo_alpha_to_1}
\end{figure}

Next, we ask whether we can design landmark distribution and choose $\alpha$ appropriately so that $\alpha$-LAD can recover the original AD. Consider the same $\mathbb{S}^1$ model with $\Phi_1:\mathbb{S}^1\rightarrow E$, but with the data sampling distribution $p^{(2)}$ non-uniform. According to the established theorem, this can be achieved if $p_\Z^{(2)}= p^{(2)}$ and $\alpha=0.5$. The sample point follows the probability density function $p^{(2)}(\theta)=\frac{1}{\pi}[\tan^{-1}(\frac{1}{2}\tan\theta)]$. We randomly select 2500 points $\{(s_i,r_i):s_i\in E, r_i\in\mathbb{S}^1\}$ and choose 1000 landmark points from them. 
The quantification results with various $\alpha$ are illustrated in Figure \ref{fig:evalue_evect_alpha}. We observe that $1/2$-LAD approximates the embedding by AD. This result is consistent with our theoretical expectations.

\begin{figure}[hbt!]
\begin{center}
\includegraphics[trim=100 0 100 0, width=1.0\textwidth]{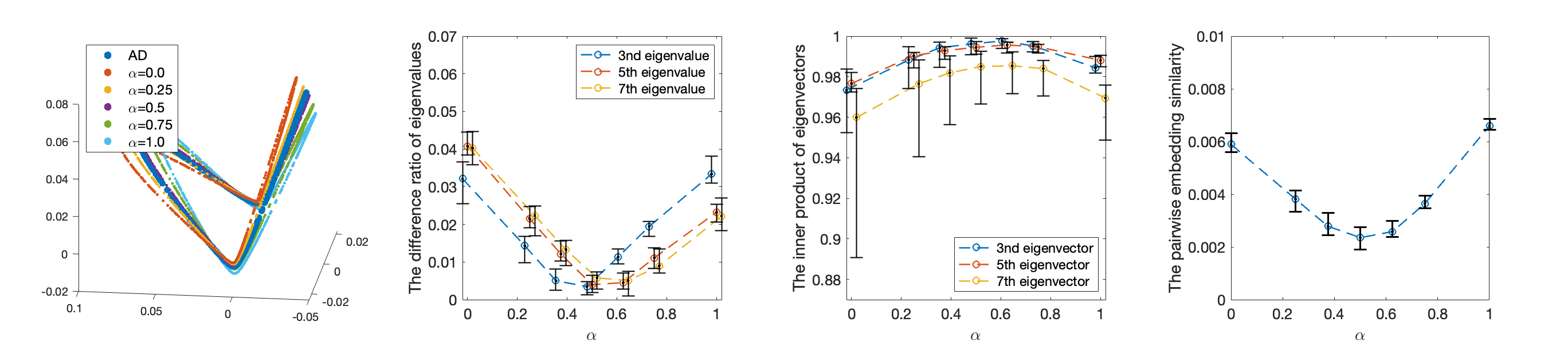}
\end{center}
\caption{Left: The $\alpha$-LAD with $p_\Z^{(2)}=p^{(2)}$ and $\alpha=0,1/4,1/2,3/4,1$. Middle left: The difference ratio of the 2nd, 4th, 6th non-trivial eigenvalues of AD and $\alpha$-LAD with $\alpha=0,0.25,0.375,0.5,0.625,0.75,1$. Middle right: The inner product of the 2nd, 4th, 6th non-trivial eigenvectors of AD and $\alpha$-LAD with $\alpha=0,0.25,0.375,0.5,0.625,0.75,1$. Right: The embedding similarity between AD and $\alpha$-LAD with $q=6$.}
\label{fig:evalue_evect_alpha}
\end{figure}


\section{Application to Automatic Sleep Stage Annotation}
\label{sec:sleep}
According to the American Academy of Sleep Medicine (AASM), sleep dynamics can be classified into two broad stages: rapid eye movement (REM) and non-rapid eye movement (NREM). The NREM stage is further divided into N1, N2 and N3. To assess the sleep stage, experts read electroencephalogram (EEG), electroocculogram (EOG), and electromyogram (EMG) following the R\&K criteria \cite{berry2015, rechtschaffen1968}.

AD has been applied to establish an automatic sleep stage annotation system from two EEG channels \cite{liu2020diffuse}. The key idea is that the common variable among various channels reveals a more reliable physiological process associated with sleep stages by eliminating sensor-specific disturbances. However, the computational burden limits the application of AD to short signals. We now demonstrate that LAD is a computationally efficient alternative to AD for extracting common sources of brain dynamics recorded by the O1-A2 and O2-A1 electrodes for an automatic sleep stage annotation system for five sleep stages, including awake, REM, N1, N2, and N3.
We consider polysomnogram data from forty subjects without sleep apnea from the Taiwan Integrated Database for Intelligent Sleep (TIDIS). All subjects' recordings exceed six hours. All experiments were conducted on a computer equipped with two 12-core CPUs, Intel(R) Xeon(R) CPU E5-2697 v2, utilizing MATLAB software version 9.12 (R2022a) in its 64-bit configuration for experimental design and analysis.

\subsection{Feature Extraction}
By the R\&K criteria \cite{berry2015, rechtschaffen1968}, we take a 30-second epoch for sleep stage evaluation. Each EEG channel recording undergoes the following preprocessing.
Suppose the $i$-th epoch starts at time $t_i$ and ends at $t_i+30$. We extend the $i$-th epoch to a 90-second EEG segment starting at $t_i-60$ and ending at $t_i+30$. Since the sampling rate is 200 Hz, each data vector is of size 18,000.
Following \cite{liu2020diffuse}, we apply the scattering transform \cite{anden2014} to extract spectral features from each 90-second EEG segment
For subject $x$, denote $u^{(\ell)}_{x,j}$ as the extracted feature of the $j$-th epoch in the $\ell$-th channel, where $\ell = 1, 2$ represents O1-A2 and O2-A1, respectively, and $x = 1, \dots, 40$ represents 40 subjects. Let $\mathcal{U}^{(\ell)} = \bigcup_{x = 1}^{40} \{u^{(\ell)}_{x,j}\}_{j = 1}^{J_x}$ be the scattering EEG spectral features of the $\ell$-th channel, where $J_x$ is the number of epochs of subject $x$. In this dataset, the size of $\mathcal{U}^{(\ell)}$ is 27,090 for $\ell = 1, 2$.
Next, we apply AD or LAD to fuse information from the two channels. Let $m = 850\approx 5\sqrt{n}$ is the size of the landmark set. Furthermore, following the established theorem and the above simulation results, we set $\alpha=0.5$ and uniformly sample landmark set from $\{(s_i,r_i)\}$, \textit{i.e.}, $p^{(2)}=p^{(2)}_\Z$ and the landmark set is a subset of the original database.

It is well known that N2 stages dominate other sleep stages, and usually, N1 and REM are much less frequent (See the leftmost column in Table \ref{tab:conf_lead}, which show 52\% epochs are labeled N2 and only 9\% are labeled N1). While handling the imbalanced data is not the focus of this study, for the sake of completeness of the analysis, we consider a simple balancing scheme -- we uniformly sample 170 epochs from each stage, resulting in a total of 850 epochs, as the landmark set. We conducted separate experiments with and without the class balancing. In the subsequent experiments, $0.5$-LAD without a class balancing is denoted by $0.5$-LAD*.

\subsection{Classification}
We chose the standard and widely used kernel support vector machine (SVM) \cite{steinwart2008support} for the learning step. Kernel SVM finds a nonlinear hyperplane to separate the data set into two disjoint subsets. Specifically, we opted for the radial basis function (RBF) kernel. Since our sleep dynamics classification problem involves five classes, we needed to extend the kernel SVM to a multi-class SVM. To achieve this, we applied the one-versus-all (OVA) classification scheme.
We balance the dataset using random down-sampling, which involves randomly selecting epochs from the majority class (N2) and deleting them from the training dataset before kernel SVM training. In our experiment, we randomly delete 50\% of N2 epochs before training.

\subsection{Results}


After applying $0.5$-LAD, $0.5$-LAD*, and AD to 29,070 epochs, the top 3 non-trivial eigenvectors are shown in Figure~\ref{fig:dimred}. It can be observed that the results of $0.5$-LAD and $0.5$-LAD* are {\em nearly} identical to those of AD. The inner product of the top 3 non-trivial eigenvectors between $0.5$-LAD and $0.5$-LAD* are 0.9940, 0.9937, and 0.9936, respectively and the pairwise embedding similarity is 0.006.

We conducted leave-one-subject-out cross-validation (LOSOCV) to assess the SVM prediction performance. The accuracy of AD, $0.5$-LAD and $0.5$-LAD* are $0.793\pm 0.052$, $0.780\pm 0.035$ and $0.785\pm 0.039$ over LOSOCV, respectively. The macro F1 of AD, $0.5$-LAD and $0.5$-LAD* are $0.655\pm 0.068$, $0.650\pm 0.050$ and $0.644\pm 0.058$ over LOSOCV, respectively. By considering a p-value less than 0.05 as indicating statistical significance, both methods showed no significant difference in accuracy or macro F1 compared to AD using the Wilcoxon sign-rank test with Bonferroni correction (for accuracy, the $p$ values are $0.459$ and $0.863$ when comparing $0.5$-LAD and $0.5$-LAD* with AD respectively, and for macro F1, the $p$ values are $0.338$ and $0.718$ when comparing $0.5$-LAD and $0.5$-LAD* with AD respectively).
%
Notably, the computational times for $0.5$-LAD and $0.5$-LAD* are $14.4$ and $13.8$ seconds respectively, which is shorter by over an order of magnitude than AD's average computation time, which is $9$ minutes. In summary, despite slightly lower performance without statistical significance, adopting $0.5$-LAD substantially reduces the computational time, enhancing practical utility.

The distributions of accuracy and macro-F1 by LOSOCV obtained from different methods are depicted on the left-hand side and right-hand side of Figure~\ref{fig:perfStage}, respectively. Through LOSOCV, we acquired 40 confusion matrices for each validation subject. It is evident that the outcomes of $0.5$-LAD closely resemble those of AD, while significantly reducing computational time. The cumulative sum of all confusion matrices derived from AD and $0.5$-LAD are depicted in Table \ref{tab:conf_lead}. Notably, the recall of $0.5$-LAD* outperforms that of $0.5$-LAD in the classes with fewer data, namely REM and N1, which emphasizes the importance of balancing the data.

\begin{figure}[hbt!]
\begin{center}
\includegraphics[trim=250 0 250 0, width=1.\textwidth]{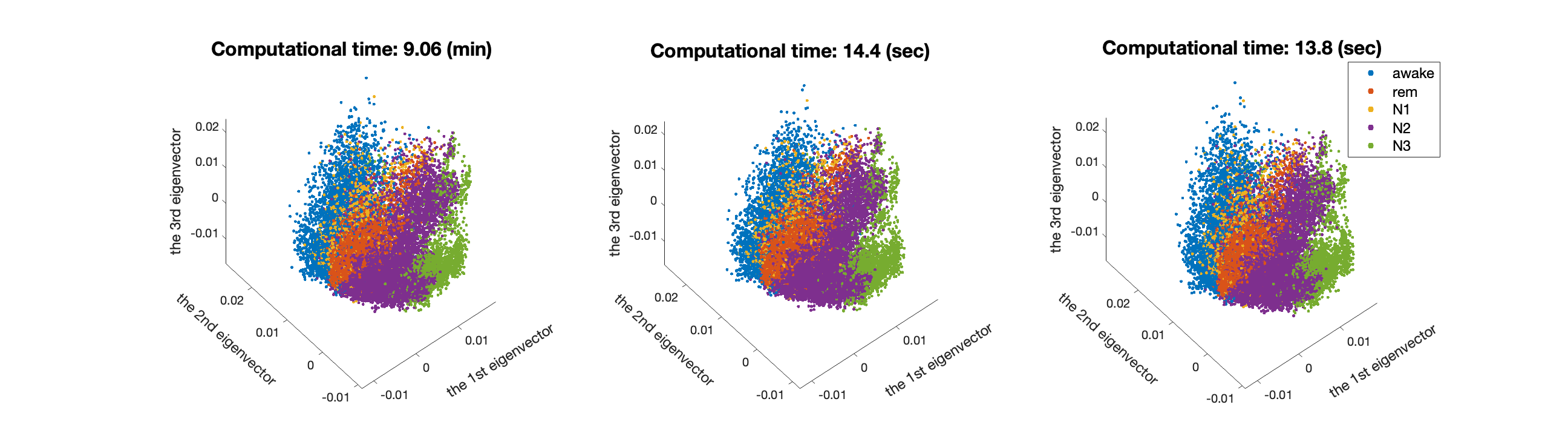}
\end{center}
\caption{Left: embedding by AD. Middle: embedding by $0.5$-LAD. Right: embedding by $0.5$-LAD*.}
\label{fig:dimred}
\end{figure}

\begin{figure}[hbt!]
\begin{center}
\includegraphics[trim=0 30 0 0, clip,width=0.9\textwidth]{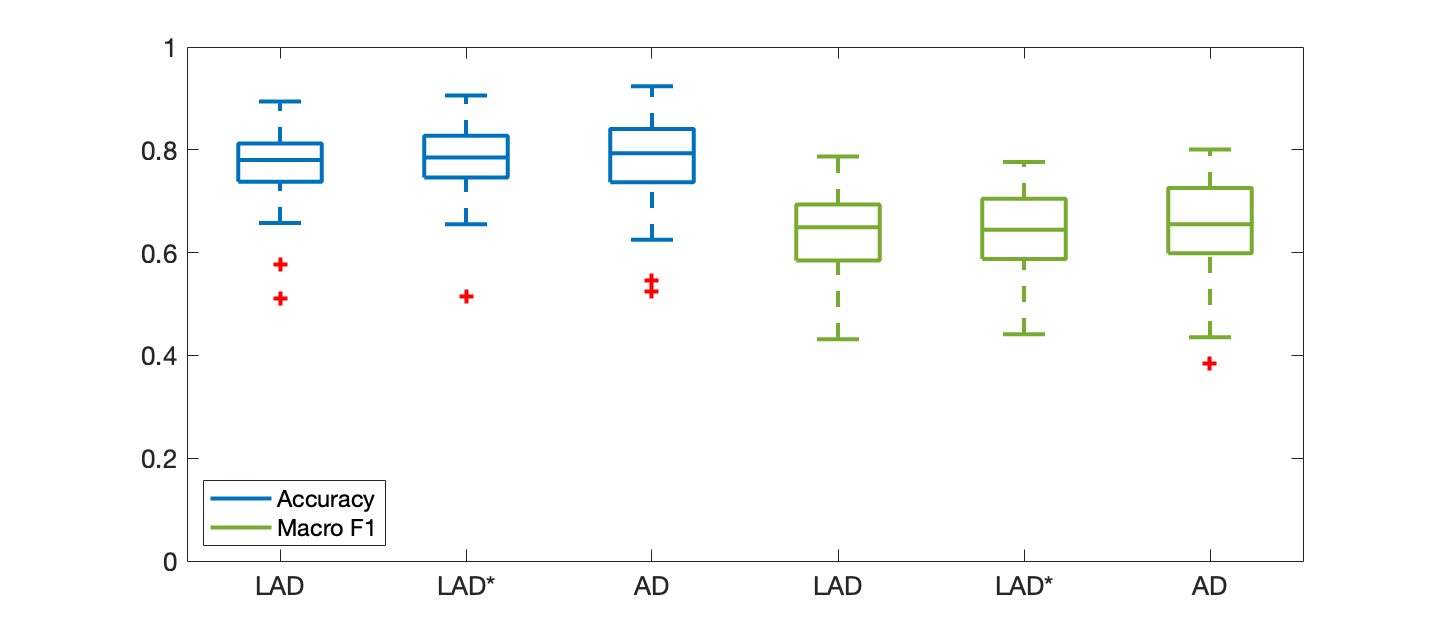}
\end{center}
\caption{The box plot of 40-fold validation accuracies by different methods. LAD stands for Landmark AD without class-balanced landmark set fusion 2 channels; LAD* stands for Landmark AD with class-balanced landmark set fusion 2 channels AD stands for AD fusion 2 channels and O1A2, O2A1 stands for diffusion map on single channels.}
\label{fig:perfStage}
\end{figure}

\begin{table}[hbt!]
\caption{Upper: The cumulative sum of all confusion matrices of AD. Middle: The cumulative sum of all confusion matrices of $0.5$-LAD without class-balanced landmark set. Lower: The cumulative sum of all confusion matrices of $0.5$-LAD* with class-balanced landmark set.}
\small
\begin{tabular}
{!{\vrule width 1.5pt}
>{\centering\arraybackslash}m{0.3cm} |
>{\centering\arraybackslash}m{2.3cm} |
>{\centering\arraybackslash}m{1.9cm} |
>{\centering\arraybackslash}m{1.9cm} |
>{\centering\arraybackslash}m{1.9cm} |
>{\centering\arraybackslash}m{1.9cm} |
>{\center \arraybackslash}m{1.9cm}
!{\vrule width 1.5pt}}
\noalign{\hrule height 1.5pt}
\multicolumn{2}{!{\vrule width 1.5pt}c|}{\multirow{2}{*}{AD}} & \multicolumn{5}{c!{\vrule width 1.5pt}}{Predict}\\\cline{3-7}
\multicolumn{2}{!{\vrule width 1.5pt}c|}{}& Awake & REM & N1 & N2 & N3\\\cline{2-6}
\noalign{\hrule height 1.5pt}
\multirow{5}{*}{\rotatebox{90}{True}}&Awake (16 \%) &3477 (82\%) & 143 (3\%) & 370 (9\%) & 236 (6\%) & 4 (0\%)\\
&REM (13 \%) &63 (2\%) & 2517 (72\%) & 514 (15\%) & '414 (12\%) & 1 (0\%)\\
&N1 (9 \%) & 432 (18\%) & 581 (24\%) & 868 (36\%) & 548 (23\%) & 0 (0\%)\\
&N2 (52 \%) & 287 (2\%) & 478 (3\%) & 464 (3\%) & 12441 (88\%) & 523 (4\%)\\
&N3 (10 \%) & 121 (4\%) & 2 (0\%) & 2 (0\%) & 844 (31\%) & 1760 (64\%)\\
\noalign{\hrule height 1.5pt}
\multicolumn{2}{!{\vrule width 1.5pt}c|}{\multirow{2}{*}{$0.5$-LAD}} & \multicolumn{5}{c!{\vrule width 1.5pt}}{Predict}\\\cline{3-7}
\multicolumn{2}{!{\vrule width 1.5pt}c|}{}& Awake & REM & N1 & N2 & N3\\
\noalign{\hrule height 1.5pt}
\multirow{5}{*}{\rotatebox{90}{True}}&Awake (16 \%) &3472 (82\%) & 173 (4\%) & 356 (8\%) & 226 (5\%) & 3 (0\%)\\
&REM (13 \%) & 90 (3\%) & 2475 (71\%) & 461 (13\%) & 482 (14\%) & 1 (0\%)\\
&N1 (9 \%) & 470 (19\%) & 620 (26\%) & 742 (31\%) & 597 (25\%) & 0 (0\%)\\
&N2 (52 \%) & 276 (2\%) & 485 (3\%) & 447 (3\%) & 12422 (88\%) & 563 (4\%)\\
&N3 (10 \%) & 30 (1\%) & 0 (0\%) & 2 (0\%) & 978 (36\%) & 1719 (63\%)\\
\noalign{\hrule height 1.5pt}
\multicolumn{2}{!{\vrule width 1.5pt}c|}{\multirow{2}{*}{$0.5$-LAD*}} & \multicolumn{5}{c!{\vrule width 1.5pt}}{Predict}\\\cline{3-7}
\multicolumn{2}{!{\vrule width 1.5pt}c|}{}& Awake & REM & N1 & N2 & N3\\
\noalign{\hrule height 1.5pt}
\multirow{5}{*}{\rotatebox{90}{True}}&Awake (16 \%) &3555 (84\%) & 163 (4\%) & 291 (7\%) & 216 (5\%) & 5 (0\%)\\
&REM (13 \%) &94 (3\%) & 2499 (71\%) & 439 (13\%) & 476 (14\%) & 1 (0\%)\\
&N1 (9 \%) &508 (21\%) & 592 (24\%) & 774 (32\%) & 554 (23\%) & 1 (0\%)\\
&N2 (52 \%) & 280 (2\%) &506 (3\%) & 470 (3\%) & 12431 (88\%) & 506 (4\%)\\
&N3 (10 \%) & 24 (1\%) & 3 (0\%) & 14 (1\%) & 855 (31\%) & 1833 (67\%)\\
\noalign{\hrule height 1.5pt}
\end{tabular}
\label{tab:conf_lead}
\end{table}

\clearpage
\section{Proof of the Main Theorem}
\label{sec:proof}

With the preparation in Section \ref{sec:main_thm}, in this section, we will prove Theorem \ref{thm:main} and \ref{thm:var}.

\subsection{Some background material for the proof}
We collect some lemmas that are needed for the proof in this subsection.

\begin{lemma}[\cite{talmon2019} Lemma A.2]
In the normal coordinate around $x \in \mathcal{M}$, when $\|v\|_{g} \ll 1, v \in T_x \mathcal{M}$, the Riemannian measure satisfies
\begin{equation*}
\mathrm{d} V\left(\exp _x v\right)=\left(1-\frac{1}{6} \operatorname{Ric}_{k l}(x) v^k v^l+O\left(\|v\|^3\right)\right) \mathrm{d} v^1 \wedge \mathrm{d} v^2 \ldots \wedge \mathrm{d} v^n
\end{equation*}
where $\operatorname{Ric}$ is the Ricci curvature.
\end{lemma}

First, we evaluate the ambient distance of two points by the same metric.
\begin{lemma}[\cite{talmon2019} Lemma A.2]
\label{lem:same_metric}
Fix $x \in \mathcal{M}$ and $y=\exp _x^{(2)}(v)$, where $v \in T_x \mathcal{M}$ with $\|v\|_{g^{(2)}} \ll 1$. We have
\begin{equation*}
\iota^{(2)}(y)=\iota^{(2)}(x)+\mathrm{d} \iota^{(2)}(v)+Q_2^{(2)}(v)+Q_3^{(2)}(v)+O\left(\|v\|_{g^{(i)}}^4\right),
\end{equation*}
where $\Pi^{(2)}$ is the second fundamental form of $\iota^{(2)}, Q_2^{(2)}(v):=\frac{1}{2} \Pi^{(2)}(v, v)$ and $Q_3^{(2)}(v):=\frac{2}{6} \nabla_v^{(2)} \Pi^{(2)}(v, v)$. Further, for $z=\exp _x^{(2)}(u)$, where $u \in T_x \mathcal{M}$ with $\|u\|_{g^{(2)}} \ll 1$, we have
\begin{equation*}
\begin{aligned}
\left\|\iota^{(2)}(z)-\iota^{(2)}(y)\right\|= & \|u-v\|_{g^{(2)}}+\tilde{Q}_2^{(2)}(u,v)+\tilde{Q}_3^{(2)}(u,v)+O\left(\|u\|_{g^{(2)}}^4,\|v\|_{g^{(2)}}^4\right)
\end{aligned}
\end{equation*}
where
\begin{align*}
\tilde{Q}_2^{(2)}(u,v)&=\frac{\left\|Q_2^{(2)}(u)-Q_2^{(2)}(v)\right\|^2}{2\|u-v\|_{g^{(2)}}} \\
\tilde{Q}_3^{(2)}(u,v)&=\frac{\left\langle\mathrm{d} \iota^{(2)}(u-v), Q_3^{(2)}(u)-Q_3^{(2)}(v)\right\rangle}{\|u-v\|_{g^{(2)}}}\,.
\end{align*}
\end{lemma}

Recall that we could find a diffeomorphism map $\Phi:\mathbb{R}^d\rightarrow\mathbb{R}^d$ such that $\Phi\circ \iota^{(2)}=\iota^{(1)}$. Its behavior is summarized below. Now, we evaluate the ambient distance by the differnet metric.
\begin{lemma}[\cite{talmon2019} Lemma A.3]
\label{lem:diff_metric}
Fix $x \in \mathcal{M}$. To simplify the notation, we ignore the subscription of $\left.\nabla^{(2)} \Phi\right|_{\iota^{(2)}(x)}$ and use $\nabla \Phi$. Similar simplification holds for $\nabla^{(2)^2} \Phi, \mathrm{d} \iota^{(2)}, \Pi^{(2)}$, etc. Suppose $y=\exp _x^{(2)} u$, where $u \in T_x \mathcal{M}$ and $\|u\|_{g^{(2)}}$ is small enough. Then we have
\begin{equation*}
\iota^{(1)}(y)=\iota^{(1)}(x)+\nabla \Phi \mathrm{d} \iota^{(2)} u+Q_2^{(1)}(u)+Q_3^{(1)}(u)+O\left(\|u\|_{g^{(2)}}^4\right)
\end{equation*}
where $Q_2^{(1)}$ and $Q_3^{(1)}$ are quadratic and cubic polynomials respectively defined by $Q_2^{(1)}(u):=\frac{1}{2} \Pi^{(1)}(u, u)$ and $Q_3^{(1)}(u):=\frac{1}{6} \nabla \Phi \nabla_u^{(2)} \Pi^{(2)}(u, u)+\nabla^2 \Phi\left(u, \Pi^{(2)}(u, u)\right)$.
Further, for $x\in \M$ defined as above, we have
\begin{equation*}
\big\|\iota^{(1)}(x)-\iota^{(1)}(y)\big\|=\big\|\nabla \Phi \mathrm{d} \iota^{(2)}(u)\big\|+\tilde{Q}_3^{(1)}(u)+O\left(\|u\|^4\right),
\end{equation*}
where
\begin{equation*}
\tilde{Q}_3^{(1)}(u):=\frac{2\left\langle\nabla \Phi \mathrm{d} \iota^{(2)}(u), Q_3^{(1)}(u)\right\rangle+\left\|Q_2^{(1)}(u)\right\|^2}{\left\|\nabla \Phi \mathrm{d} \iota^{(2)}(u)\right\|} .
\end{equation*}
\end{lemma}


We also need the following truncation lemma. 
\begin{lemma}[\cite{talmon2019} Lemma A.5]
\label{lem:trunc}
Suppose Assumption \ref{asp:assumption} holds, $F \in L^{\infty}(\mathcal{M})$ and $0<\gamma<1 / 2$. Then, when $\epsilon$ is small enough, for all pairs of $x, y \in \mathcal{M}$ so that $y=\exp _x^{(2)} v$ and $\|v\|_{g^{(2)}}>2 \epsilon^\gamma$, the following holds:
    \begin{equation*}
    \left|\int_{\mathcal{M}} K_\epsilon^{\left(1\right)}(x, z) K_\epsilon^{\left(2\right)}\left(z, y\right) F(z) \mathrm{d} V^{(2)}(z)\right|=O\left(\epsilon^{d / 2+3 / 2}\right)\,.
    \end{equation*}
\end{lemma}

\subsection{Proof of Theorem \ref{thm:main}}

We start with the asymptotical behavior of the kernel associated with LAD. First, for convenience, denote
\begin{equation*}
\tilde{B}_h^{(2)}(x):=\exp _x^{(2)}\left(B_h^{(2)}\right),
\end{equation*}
where $B_h^{(2)}=\left\{u \in T_x \mathcal{M} \mid\|u\|_{g^{(2)}} \leq h\right\} \subset T_x \mathcal{M}$ is a $d$-dim disk with the center 0 and the radius $h>0$.
\begin{lemma}
\label{lem:refkernel}
Take $f \in C^3(\mathcal{M})$, $0<\gamma<1 / 2$ and $x, y \in \mathcal{M}$ so that $y=\exp _x^{(2)} v$, where $v \in T_x \mathcal{M}$ and $\|v\|_{g^{(2)}} \leq 2 \epsilon^\gamma$. Fix normal coordinates around $x$ associated with $g^{(1)}$ and $g^{(2)}$ and fix orthonormal coordinates associated with $g^{(1)}$ and $g^{(2)}$. Under these orthonormal coordinates, denote 
\begin{align}
R_x:=\big[\mathrm{d} \exp _x^{(2)}\big|_0\big]^{-1}\big[\mathrm{d} \iota^{(2)}\big]^{-1} \nabla \Phi\big[\mathrm{d} \iota^{(1)}\big]\big[\mathrm{d} \exp _x^{(1)}\big|_0\big]\,,\label{definition Rx}
\end{align}
which maps $\mathbb{R}^d$ to $\mathbb{R}^d$ associated with the basis $\{E_i\}_{i=1}^d$. Then, when $\epsilon$ is sufficiently small, the following holds:
\begin{equation*}
\begin{aligned}
& \int_{\mathcal{M}} K_\epsilon^{(1)}(x, z) K_\epsilon^{(2)}\left(z, y\right) F(z) \mathrm{d} V^{(2)}(z) \\
= &\, \epsilon^{d / 2}\big[F(x) A_{0,\epsilon}(v)+\epsilon^{1/2} A_{1,\epsilon}(F, v)+\epsilon A_{2,\epsilon}(F, v)\big]+O(\epsilon^{d/2+3 / 2})\,,
\end{aligned}
\end{equation*}
where
\begin{align*}
A_{0,\epsilon}(v)&\,:=\int_{\mathbb{R}^d} \tilde{K}^{(1)}\left(\left\|R_x w\right\|\right) \tilde{K}^{(2)}(\|w-v / \sqrt{\epsilon}\|) \mathrm{d} w\,,\\
A_{1, \epsilon}(F, v)&\,:=\int_{\mathbb{R}^d} \tilde{K}^{(1)}\left(\left\|R_x w\right\|\right) \tilde{K}^{(2)}(\|w-v / \sqrt{\epsilon}\|) \nabla_w^{(2)} F(x) \mathrm{d} w\,,\\
A_{2, \epsilon}(F, v)&\,:=F(x)B_{21,\epsilon}(v)+F(x)B_{20,\epsilon}(v)+B_{22,\epsilon}(F, v)\nonumber\,,
\end{align*}
where
\begin{align*}
B_{21,\epsilon}\,(v):= &\int_{\mathbb{R}^d}\big[\tilde{K}^{(1)}\big]^{\prime}\left(\left\|R_x w\right\|\right) \tilde{K}^{(2)}(\|w-v / \sqrt{\epsilon}\|) \tilde{Q}_3^{(2)}(u) \mathrm{d} w \\
 &\hspace*{0.5cm}+\int_{\mathbb{R}^d} \tilde{K}^{(1)}\left(\left\|R_x w\right\|\right)\big[\tilde{K}^{(2)}\big]^{\prime}(\|w-v / \sqrt{\epsilon}\|) \tilde{Q}_3^{(2)}(w, v / \sqrt{\epsilon}) \mathrm{d} w \nonumber\\
B_{20,\epsilon}(v):= &\int_{\mathbb{R}^d} \tilde{K}^{(1)}\left(\left\|R_x w\right\|\right) \tilde{K}^{(2)}(\|w-v / \sqrt{\epsilon}\|) \operatorname{Ric}_{i j}^{(1)}(x) w^i w^j \mathrm{~d} w \,,\\
B_{22,\epsilon}(F, v):=&\int_{\mathbb{R}^d} \tilde{K}^{(1)}\left(\left\|R_x w\right\|\right) \tilde{K}^{(2)}(\|w-v / \sqrt{\epsilon}\|) \frac{\nabla_{w, w}^{(2)}{ }^2 F(x)}{2} \mathrm{d} w \,,
\end{align*}
where $\operatorname{Ric}^{(\ell)}$ is the Ricci curvature and $\Pi^{(\ell)}$ is the second fundamental associated with $g^{(\ell)}$. Note that $B_{21,\epsilon}(v)$ depends on the first derivative of kernel functions, $B_{20,\epsilon}(v)$ depends on Ricci curvature, and $B_{22,\epsilon}(F,v)$ depends on the second derivative of $F$. Moreover, $A_{0,\epsilon}(v)=A_{0,\epsilon}(-v)$, $A_{1,\epsilon}(F,v)=-A_{1,\epsilon}(F,-v)$ and $A_{2,\epsilon}(F,v)=A_{2,\epsilon}(F,-v)$.
\end{lemma}
\begin{proof}
By the definition of kernels, 
\begin{equation*}
\begin{aligned}
& \int_{\M} K_\epsilon^{(1)}(x, z) K_\epsilon^{(2)}\left(z,y\right) F(z) \mathrm{d} V^{(2)}(z) \\
= & \int_{\M} \tilde{K}^{(1)}\left(\frac{\left\|\iota^{(1)}(x)-\iota^{(1)}(z)\right\|}{\sqrt{\epsilon}}\right) \tilde{K}^{(2)}\left(\frac{\left\|\iota^{(2)}(z)-\iota^{(2)}\left(y\right)\right\|}{\sqrt{\epsilon}}\right) F(z) \mathrm{d} V^{(2)}(z) \,.
\end{aligned}
\end{equation*}
Suppose $z=\exp_x^{(2)}(u)$, where $u\in T_x\M$ with $\|u\|_{g^{(2)}}\ll 1$. By Lemma \ref{lem:diff_metric}, 
\begin{equation*}
\begin{aligned}
& \tilde{K}^{(1)}\left(\frac{\left\|\iota^{(1)}(x)-\iota^{(1)}(z)\right\|}{\sqrt{\epsilon}}\right) \\
= \, & \tilde{K}^{(1)}\left(\frac{\left\|\nabla \Phi \mathrm{d} \iota^{(2)}(u)\right\|}{\sqrt{\epsilon}}\right)+\big[\tilde{K}^{(1)}\big]^{\prime}\left(\frac{\left\|\nabla \Phi \mathrm{d} \iota^{(2)}(u)\right\|}{\sqrt{\epsilon}}\right) \frac{\tilde{Q}_3^{(1)}(u)}{\sqrt{\epsilon}}+O\left(\frac{\|u\|^5}{\epsilon}\right)
\end{aligned}
\end{equation*}
Note that $\tilde{Q}_3^{(1)}(u)$ is an even function. On the other hand, by Lemma \ref{lem:same_metric},
\begin{equation*}
\begin{aligned}
& \tilde{K}^{(2)}\left(\frac{\left\|\iota^{(2)}(y)-\iota^{(2)}\left(x^{\prime \prime}\right)\right\|}{\sqrt{\epsilon}}\right)\\
=\,&\tilde{K}^{(2)}\left(\frac{\left\|\mathrm{d} \iota^{(2)}(u-v)\right\|}{\sqrt{\epsilon}}\right) +\big[\tilde{K}^{(2)}\big]^{\prime}\left(\frac{\left\|\mathrm{d} \iota^{(2)}(u-v)\right\|}{\sqrt{\epsilon}}\right) \frac{\tilde{Q}_3^{(2)}(u, v)}{\sqrt{\epsilon}}+O\left(\frac{\|u\|^4}{\sqrt{\epsilon}}, \frac{\|v\|^4}{\sqrt{\epsilon}}\right) .
\end{aligned}
\end{equation*}
Note that $\tilde{Q}_3^{(2)}(-u, -v)=-\tilde{Q}_3^{(2)}(u, v)$. By the Lemma \ref{lem:trunc}, exist a $\gamma\in(0,1)$ such that we could replace the integral domain $\tilde{B}^{(2)}_{\epsilon^\gamma}(x)\subset T_x\M$ by $T_x\M$ with order $O(\epsilon^{d/2+3/2})$ where $\epsilon$ is small enough. Now, it is sufficient to apply the Taylor expansion,
\begin{equation*}
\begin{aligned}
&\int_{\mathcal{M}} K_\epsilon^{(1)}(x, z) K_\epsilon^{(2)}\left(z, y\right) F(z) \mathrm{d} V^{(2)}(z)\\
=\,& \int_{\mathbb{R}^d} \left[\tilde{K}^{(1)}\left(\frac{\left\|R_x u\right\|}{\sqrt{\epsilon}}\right)+\big[\tilde{K}^{(1)}\big]^{\prime}\left(\frac{\left\|R_x u\right\|}{\sqrt{\epsilon}}\right) \frac{\tilde{Q}_3^{(2)}(u)}{\sqrt{\epsilon}}+O\left(\frac{\|u\|^5}{\epsilon}\right)\right]  \\
&\quad \times\left[\tilde{K}^{(1)}\left(\frac{\|u-v\|}{\sqrt{\epsilon}}\right)+\big[\tilde{K}^{(2)}\big]^{\prime}\left(\frac{\|u-v\|}{\sqrt{\epsilon}}\right) \frac{\tilde{Q}_3^{(2)}(u, v)}{\sqrt{\epsilon}}+O\left(\frac{\|u\|^4}{\sqrt{\epsilon}}, \frac{\|v\|^4}{\sqrt{\epsilon}}\right)\right] \\
&\quad \times\left[F(x)+\nabla_u^{(1)} F(x)+\frac{\nabla_{u, u}^{(1)2} F(x)}{2}+O(\|u\|^3)\right] \\
&\quad \times\left[1-\operatorname{Ric}_{i j}^{(1)}(x) u^i u^j+O\left(\|u\|^3\right)\right] \mathrm{d} u+O(\epsilon^{d / 2+3 / 2}) \\
=\,&\epsilon^{d / 2}\big[F(x) A_{0, \epsilon}(v)+\epsilon^{1 / 2} A_{1, \epsilon}(F, v)+\epsilon A_{2, \epsilon}(F, v)+O(\epsilon^{3 / 2})\big],
\end{aligned}
\end{equation*}
Let $w=u/\sqrt{\epsilon}$. By the change of variables, the result follows.
\end{proof}

We remark that if the chosen kernels are both Gaussian, the $\alpha$-landmark alternative kernel is Gaussian. Specifically, when $\tilde{K}^{(1)}$ and $\tilde{K}^{(2)}$ are both Gaussian, that is, $\tilde{K}^{(1)}(t)=\tilde{K}^{(2)}(t)=e^{-t^2} / \sqrt{\pi}$, we have
\begin{equation}\label{eq:lanker_decay}
A_0(v)=\frac{\pi^{d / 2}}{\sqrt{\operatorname{det}\left(I+\Lambda_x^2\right)}} e^{-\left\|\left(I+\Lambda_x^2\right)^{-1 / 2} \Lambda_x V_x^T v\right\|^2 / \epsilon} \,,
\end{equation}
which satisfies the exponential decay property of the kernel functions.

\begin{lemma}
    \label{lem:double_kernel}
Fix $x \in \mathcal{M}$ and pick $F \in C^3(\mathcal{M})$. Fix normal coordinates around $x$ associated with $g^{(1)}$ and $g^{(2)}$ and denote $\{E_i\}_{i=1}^d \subset T_x \mathcal{M}$ to be an orthonormal basis associated with $g^{(1)}$. Set $R_x$ like \eqref{definition Rx} and by the SVD $R_x=U_x \Lambda_x V_x^T$, where $\Lambda_x=\operatorname{diag}\left[\lambda_1, \ldots, \lambda_d\right]$. Then, when $\epsilon$ is sufficiently small, we have
\begin{align*}
&\int_{\mathcal{M}} \int_{\mathcal{M}} K_\epsilon^{(1)}(x, z) K_\epsilon^{(2)}\left(z, y\right) F(z) \mathrm{d} V^{(2)}(z) G\left(y\right) \mathrm{d} V^{(2)}\left(y\right)\\
=&\,\epsilon^d\left[\frac{G(x)F(x)}{\operatorname{det}\left(\Lambda_x\right)}+\epsilon C^F(G)\right]+O(\epsilon^{d+3/2})\,,
\end{align*}
where
\begin{align*}
C^F(G):=F(x)G(x)W(x)+C^F_{1}(G)+C^F_{2}(G)
\end{align*}
where
\begin{align*}
W(x)&=\frac{2\mu_{2,0}^{(1)}}{d \operatorname{det}\left(\Lambda_x\right)}\left( \sum_{i=1}^d \lambda_i\operatorname{Ric}_{i i}^{(2)}(x)\right)+\frac{\mu_{2,0}^{(2)}}{d \operatorname{det}\left(\Lambda_x\right)}s^{(2)}(x)\\
&\hspace*{1cm}+\int_{\mathbb{R}^d}B_{21,\epsilon}(\sqrt{\epsilon}v)\mathrm{d}v
\end{align*}
\begin{align*}
C^F_1(G)=\frac{\mu_{2,0}^{(1)}}{d\cdot \operatorname{det}\left(\Lambda_x\right)}\left( \sum_{i=1}^d \lambda_i\left[\nabla^{(2)}_{E_i} F(x)\nabla^{(2)}_{E_i} G(x)\right]\right) \,,
\end{align*}
and
\begin{align*}
C^F_2(G)&= \frac{\mu_{2,0}^{(1)}}{2d\cdot \operatorname{det}\left(\Lambda_x\right)}F(x)\left( \sum_{i=1}^d \lambda_i\nabla_{E_i, E_i}^{(2)^2} G(x)\right)+\frac{\mu_{2,0}^{(2)}}{2d\cdot \operatorname{det}\left(\Lambda_x\right)}F(x)\Delta^{(2)} G(x)\\
&\hspace*{1cm}+\frac{\mu_{2,0}^{(1)}}{2d\cdot \operatorname{det}\left(\Lambda_x\right)}G(x)\left( \sum_{i=1}^d \lambda_i\nabla_{E_i, E_i}^{(2)^2} F(x)\right)\,.
\end{align*}
where $\operatorname{Ric}^{(\ell)}$ is the Ricci curvature and $s^{(\ell)}$ is the scalar curvature associated with $g^{(\ell)}$. Moreover, $C^F_{1}(G)$ depends on the first derivative of $F$ and $G$, $C^F_{2}(G)$ depends on the second derivative of $F$ or $G$ and $W$ depends on the Ricci curvatures, scalar curvatures and the second fundamental form.
\end{lemma}

\begin{proof}
By Lemma \ref{lem:trunc}, let $y=\exp^{(2)}_x(v)$, where $v\in T_x\M$ and we directly compute
\begin{align*}
&\int_{\mathcal{M}} \int_{\mathcal{M}} \tilde{K}_\epsilon^{(1)}(x, z) \tilde{K}_\epsilon^{(2)}\left(z, y\right) F(z) \mathrm{d} V^{(2)}(z) G\left(y\right) \mathrm{d} V^{(2)}\left(y\right)\\
=& \epsilon^{d / 2}\int_{\mathbb{R}^d} \left[F(x) A_{0,\epsilon}(v)+\epsilon^{1/2} A_{1,\epsilon}(F, v)+\epsilon A_{2,\epsilon}(F, v)+O(\epsilon^{3 / 2})\right] \\
& \times\left[G(x)+\nabla_v^{(2)} G(x)+\frac{\nabla_{v, v}^{(2)}{ }^2 G(x)}{2}+O\left(\|v\|^3\right)\right] \\
& \times\left[1-\operatorname{Ric}_{i j}^{(2)}(x) v^i v^j+O\left(\|v\|^3\right)\right] \mathrm{d} v+O\left(\epsilon^{d +3/2}\right) \\
:=\,&\epsilon^{d/2}[I_0+\epsilon^{1/2}I_1+ I_2+\epsilon I_3+ I_4]+O(\epsilon^{d+3/2})
\end{align*}
where
\begin{align*}
I_0 &=F(x)G(x)\int_{\mathbb{R}^d} A_{0,\epsilon}(v)dv, \quad\quad I_1 =\int_{\mathbb{R}^d} A_{1,\epsilon}(F, v)\nabla_v^{(2)}G(x)dv,\\
I_2&=F(x)\int_{\mathbb{R}^d} A_{0,\epsilon}(v)\frac{{\nabla_{v,v}^{(2)}}^2G(x)}{2}dv, \quad\quad I_3 =\int_{\mathbb{R}^d} A_{2,\epsilon}(F,v) G(x)dv,\\
I_4 &= F(x)G(x)\int_{\mathbb{R}^d} A_{0,\epsilon}(v)\operatorname{Ric}_{i j}^{(2)}(x) v^i v^jdv
\end{align*}
We compute the right-hand side term by term and apply the Equation (A.60) in \cite{talmon2019}. For $I_0$, let $s=w-v/\sqrt{\epsilon}$ and we obtain
\begin{align*}
F(x)G(x)\int_{\mathbb{R}^d} A_{0,\epsilon}(v)dv &=\epsilon^{d/2}F(x)G(x) \int_{\mathbb{R}^d} K^{(1)}\left(\left\|R_x w\right\|\right)\left[\int_{\mathbb{R}^d} K^{(2)}(\|w-v/\sqrt{\epsilon}\|) \mathrm{d} v\right] \mathrm{d} w\\
&=\epsilon^{d/2}\frac{F(x)G(x)}{\operatorname{det}(\Lambda_x)}
\end{align*}
For $I_1$, set $s=w-v/\sqrt{\epsilon}$ and $u=R_xw$, and obtain
\begin{align*}
&\int_{\mathbb{R}^d} A_{1,\epsilon}(F, v)\nabla_v^{(2)}G(x)dv\\
=\,&\epsilon^{d/2+1/2} \int_{\mathbb{R}^d} \tilde{K}^{(1)}\left(\left\|R_x w\right\|\right)\left[\int_{\mathbb{R}^d} \tilde{K}^{(2)}(\|w-v/\sqrt{\epsilon}\|)\nabla_w^{(2)}F(x) \nabla_v^{(2)}G(x)\mathrm{d} v\right] \mathrm{d} w\\
=\,&\epsilon^{d/2+1/2} \int_{\mathbb{R}^d} \tilde{K}^{(1)}\left(\left\|R_x w\right\|\right)\left[\int_{\mathbb{R}^d} \tilde{K}^{(2)}(\|s\|) (w-s)^\top\nabla^{(2)}G(x)\mathrm{d} s\right]w^\top\nabla^{(2)}F(x) \mathrm{d} w\\
=\,&\epsilon^{d/2+1/2} \int_{\mathbb{R}^d} \tilde{K}^{(1)}\left(\left\|R_x w\right\|\right)  w^\top\nabla^{(2)}G(x)w^\top\nabla^{(2)}F(x) \mathrm{d} w\\
=\,&\epsilon^{d/2+1/2} \int_{\mathbb{R}^d} \tilde{K}^{(1)}\left(\left\|R_x w\right\|\right)  \nabla^{(2)}G(x)^\top (w w^\top)\nabla^{(2)}F(x) \mathrm{d} w\\
=\,&\frac{\epsilon^{d/2+1/2}\mu_{2,0}^{(1)}}{d\operatorname{det}\left(\Lambda_x\right)}\left( \sum_{i=1}^d \lambda_i\left[\nabla^{(2)}_{E_i} F(x)\nabla^{(2)}_{E_i} G(x)\right]\right)\,.
\end{align*}
where the third equality holds since the symmetric property. For $I_2$, let $s=w-v/\sqrt{\epsilon}$ and $u=R_xw$, and we have
\begin{align*}
&\int_{\mathbb{R}^d} A_{0,\epsilon}(v)\frac{{\nabla_{r,r}^{(2)}}^2G(x)}{2}dv\\
=\,&\epsilon^{d/2+1}F(x)\int_{\mathbb{R}^d} \tilde{K}^{(1)}\left(\left\|R_x w\right\|\right)\left[\int_{\mathbb{R}^d} \tilde{K}^{(2)}(\|w-v/\sqrt{\epsilon}\|) \frac{{\nabla_{v,v}^{(2)}}^2G(x)}{2}\mathrm{d} v\right] \mathrm{d} w\\
=\,&\epsilon^{d/2+1}\frac{F(x)}{2}\int_{\mathbb{R}^d} \tilde{K}^{(1)}\left(\left\|R_x w\right\|\right)\left[\int_{\mathbb{R}^d} \tilde{K}^{(2)}(\|s\|) (w-s)^\top{\nabla^{(2)}}^2G(x)(w-s)\mathrm{d} s\right] \mathrm{d} w\\
=\,&\epsilon^{d/2+1}\frac{F(x)}{2}\bigg[\int_{\mathbb{R}^d} \tilde{K}^{(2)}\left(\left\|s\right\|\right)\mathrm{d} s\int_{\mathbb{R}^d} \tilde{K}^{(1)}(\|R_x w\|) w^\top{\nabla^{(2)}}^2G(x)w \mathrm{d} w\\
&\hspace*{2.cm} +\int_{\mathbb{R}^d} \tilde{K}^{(1)}\left(\left\|R_x w\right\|\right)\mathrm{d} w\int_{\mathbb{R}^d} \tilde{K}^{(2)}(\|s\|) s^\top{\nabla^{(2)}}^2G(x)s\mathrm{d} s \bigg]
\end{align*}
By the similar argument in $I_1$, we have
\begin{equation*}
I_2=\epsilon^{d/2+1}\left[\frac{\mu_{2,0}^{(1)}}{2d \operatorname{det}\left(\Lambda_x\right)}F(x)\left( \sum_{i=1}^d \lambda_i\nabla_{E_i, E_i}^{(2)} G(x)\right)+\frac{\mu_{2,0}^{(2)}}{2d \operatorname{det}\left(\Lambda_x\right)}F(x)\Delta^{(2)} G(x) \right]\,.
\end{equation*}
For $I_3$, recall that $A_{2,\epsilon}(F,v)=F(x)B_{21,\epsilon}(v)+F(x)B_{20,\epsilon}(v)+B_{22,\epsilon}(F,v)$, and hence
\begin{align*}
&\int_{\mathbb{R}^d}A_{2,\epsilon}(F,v)G(x)dv:=I_5+I_6+I_7
\end{align*}
where
\begin{align*}
I_5 &= G(x)F(x)\int_{\mathbb{R}^d}B_{21,\epsilon}(v)dv,\quad\quad I_6=G(x)F(x)\int_{\mathbb{R}^d}B_{20,\epsilon}(v)dv\\
I_7 &=G(x)\int_{\mathbb{R}^d}B_{22,\epsilon}(F,v)dv\,.
\end{align*}
The expansion of $I_5$, $I_6$, and $I_7$ is by the same direct expansion. For $I_5$,
\begin{align*}
G(x)F(x)\int_{\mathbb{R}^d}B_{21,\epsilon}(v)dv\,&=\epsilon^{d/2}G(x)F(x)\int_{\mathbb{R}^d}B_{21,\epsilon}(\sqrt{\epsilon}v)dv\\
&:=\epsilon^{d/2}G(x)F(x)W_1(x)\,.
\end{align*}
where $W_1$ depends on second fundamental form. For $I_7$, let $s=w-v/\sqrt{\epsilon}$ and $u=R_xw$, and we have
\begin{align*}
&G(x)\int_{\mathbb{R}^d}B_{22,\epsilon}(F,v)dv\\
=\,&\epsilon^{d/2+1}G(x)\int_{\mathbb{R}^d} \tilde{K}^{(1)}\left(\left\|R_x w\right\|\right)\left[\int_{\mathbb{R}^d} \tilde{K}^{(2)}(\|w-v/\sqrt{\epsilon}\|) \frac{{\nabla_{w,w}^{(2)}}^2F(x)}{2}\mathrm{d} v\right] \mathrm{d} w\\
=\,&\epsilon^{d/2+1}\frac{G(x)}{2}\int_{\mathbb{R}^d} \tilde{K}^{(2)}\left(\left\|s\right\|\right)\mathrm{d} s\int_{\mathbb{R}^d} \tilde{K}^{(1)}(\|R_x w\|) w^\top{\nabla^{(2)}}^2F(x)w \mathrm{d} w\\
=\,&\frac{\epsilon^{d/2+1}\mu_{2,0}^{(1)}}{2d \operatorname{det}\left(\Lambda_x\right)}G(x)\left( \sum_{i=1}^d \lambda_i\nabla_{E_i, E_i}^{(2)^2} F(x)\right)\,.
\end{align*}
Finally, let $s=w-v/\sqrt{\epsilon}$, and $u=R_xw$
\begin{align*}
&I_4+\epsilon I_6=F(x)G(x)\left[\int_{\mathbb{R}^d}A_{0,\epsilon}(v)\operatorname{Ric}^{(2)}_{ij}(x)v^iv^jdv+\int_{\mathbb{R}^d}B_{20}(v)dv\right]\\
=&\frac{\epsilon^{d/2+1}\mu_{2,0}^{(1)}}{d \operatorname{det}\left(\Lambda_x\right)}G(x)F(x)\left( \sum_{i=1}^d \lambda_i\operatorname{Ric}_{i i}^{(2)}(x)\right) +\frac{\epsilon^{d/2+1}\mu_{2,0}^{(2)}}{d \operatorname{det}\left(\Lambda_x\right)}G(x)F(x)s^{(2)}(x) \\
&+\frac{\epsilon^{d/2+1}\mu_{2,0}^{(1)}}{d \operatorname{det}\left(\Lambda_x\right)}G(x)F(x)\left( \sum_{i=1}^d \lambda_i\operatorname{Ric}_{i i}^{(2)}(x)\right)
:=\epsilon^{d/2+1}G(x)F(x)W_2(x)\,,
\end{align*}
where the term $I_6$ is similar to the argument in $I_2$ and $s^{(\ell)}(x)=\sum_{i=1}^d\operatorname{Ric}^{(\ell)}_{ii}(x)$ is scalar curvature. By putting everything together and setting 
\begin{align*}
    W(x):=W_1(x)+W_2(x)&=\frac{2\mu_{2,0}^{(1)}}{d \operatorname{det}\left(\Lambda_x\right)}\left( \sum_{i=1}^d \lambda_i\operatorname{Ric}_{i i}^{(2)}(x)\right)+\frac{\mu_{2,0}^{(2)}}{d \operatorname{det}\left(\Lambda_x\right)}s^{(2)}(x)\\
&\hspace*{1cm}+\int_{\mathbb{R}^d}B_{21,\epsilon}(\sqrt{\epsilon}v)\mathrm{d}v\,,
\end{align*}
the desired result follows.
\end{proof}




With the above lemma, we can prove Theorem \ref{thm:main}.
\begin{proof}
By definition, we have,
\begin{align*}
&T_{\lan, \epsilon,\alpha} f(x)\\
=\,&\frac{\int_\M K_{\lan, \epsilon,\alpha}(x, y) f(y) p^{(2)}(y) d V(y)}{\int_{\mathcal{M}} K_{\lan, \epsilon,\alpha}(x, y) p^{(2)}(y) d V(y)}\\
=\,&\frac{\int_{\mathcal{M}}\Big(\int_{\mathcal{M}} K^{(1)}_{\epsilon}(x, z)\frac{1}{d_\epsilon^{(2)}(z)^\alpha} K^{(2)}_{\epsilon}(z, y) p_\Z^{(2)}(z) dV(z)\Big) f(y) p^{(2)}(y) dV(y)}{\int_{\mathcal{M}}\Big(\int_{\mathcal{M}} K^{(1)}_{\epsilon}(x, z)\frac{1}{d_\epsilon^{(2)}(z)^\alpha} K^{(2)}_{\epsilon}(z, y) p_\Z^{(2)}(z) dV(z)\Big) p^{(2)}(y) d V(y)}\,.
\end{align*}
The numerator of the left-hand side can be expanded and organized as
\begin{align*}
&\int_{\mathcal{M}} \int_{\mathcal{M}} K_\epsilon^{(1)}(x, z) K_\epsilon^{(2)}\left(z, y\right) \frac{p^{(2)}_\Z(z)}{d_\epsilon^{(2)}(z)^\alpha} \mathrm{d} V^{(2)}(z)  f(y)p^{(2)}(y)\mathrm{d} V^{(2)}\left(y\right)\\
=\,&\epsilon^{-d\alpha}\int_{\mathcal{M}} \int_{\mathcal{M}} K_\epsilon^{(1)}(x, z) K_\epsilon^{(2)}\left(z, y\right)p^{(2)}(z)^{-\alpha}p^{(2)}_\Z(z)^{1-\alpha}\\
&\quad\times[1-\alpha\epsilon\bar{E}(z)+\mathcal{O}(\epsilon^2)]  \mathrm{d} V^{(2)}(z)  f(y)p^{(2)}(y)\mathrm{d} V^{(2)}(y)
\end{align*}
To simplify the notation, denote $q_{\alpha}(x):=p^{(2)}(x)^{-\alpha}p^{(2)}_\Z(x)^{1-\alpha}$. By Lemma \ref{lem:double_kernel}, the above equation can be expanded and organized as
\begin{equation}
\begin{aligned}
&\epsilon^{-d(1-\alpha)}\int_{\mathcal{M}} \int_{\mathcal{M}} K_\epsilon^{(1)}(x, z) K_\epsilon^{(2)}\left(z, y\right) q_{\alpha}(z)\left[1-\epsilon\alpha\bar{E}(z)\right] \mathrm{d} V^{(2)}(z)  f(y)p^{(2)}(y)\mathrm{d} V^{(2)}\left(y\right)\\
=&\frac{f(x)p^{(2)}(x)q_{\alpha}(x)}{\operatorname{det}(\Lambda_x)}+\epsilon p^{(2)}(x)q_{\alpha}(x)W(x)f(x)-\epsilon\alpha q_{\alpha}(x)p^{(2)}(x)f(x)\bar{E}(x)\\
&+\frac{\epsilon \mu_{2,0}^{(1)}}{2d\operatorname{det}(\Lambda_x)}\bigg[q_{\alpha}(x)\sum^{d}_{i=1}\lambda_i\nabla_{E_iE_i}^{(2)^2}p^{(2)}(x)+p^{(2)}(x)\sum^{d}_{i=1}\lambda_i\nabla_{E_iE_i}^{(2)^2}q_{\alpha}(x)\\
&\hspace*{1.5cm}+2\left( \sum_{i=1}^d \lambda_i\left[\nabla^{(2)}_{E_i} p^{(2)}(x)\nabla^{(2)}_{E_i} q_{\alpha}(x)\right]\right)\bigg]f(x)\\
&+\frac{\epsilon \mu_{2,0}^{(2)}}{2d\operatorname{det}(\Lambda_x)}\bigg[q_{\alpha}(x)\Delta^{(2)}p^{(2)}(x)\bigg]f(x)\\
&+\frac{\epsilon \mu_{2,0}^{(1)}}{2d\operatorname{det}(\Lambda_x)}\bigg[2q_{\alpha}(x)\left( \sum_{i=1}^d \lambda_i\left[\nabla^{(2)}_{E_i} p^{(2)}(x)\nabla^{(2)}_{E_i} f(x)\right]\right)+2p^{(2)}(x)\left( \sum_{i=1}^d \lambda_i\left[\nabla^{(2)}_{E_i} q_{\alpha}(x)\nabla^{(2)}_{E_i} f(x)\right]\right)\bigg]\\
&+\frac{\epsilon \mu_{2,0}^{(2)}}{2d\operatorname{det}(\Lambda_x)}2q_{\alpha}(x)\nabla^{(2)} p^{(2)}(x)\cdot\nabla^{(2)} f(x)\\
&+\frac{\epsilon \mu_{2,0}^{(2)}}{2d\operatorname{det}(\Lambda_x)}q_{\alpha}(x)p^{(2)}(x)\Delta^{(2)}f(x)+\frac{\epsilon \mu_{2,0}^{(1)}}{2d\operatorname{det}(\Lambda_x)}q_{\alpha}(x)p^{(2)}(x)\sum^{d}_{i=1}\lambda_i\nabla_{E_iE_i}^{(2)^2}f(x)+\mathcal{O}(\epsilon^{3/2})\,.
\end{aligned}
\label{eq:numer}
\end{equation}
The denominator of the left-hand side can be expanded and organized as
\begin{equation}
\begin{aligned}
&\epsilon^{-d(1-\alpha)}\int_{\mathcal{M}} \int_{\mathcal{M}} K_\epsilon^{(1)}(x, z) K_\epsilon^{(2)}\left(z, y\right) \frac{p^{(2)}_\Z(z)}{d_\epsilon^{(2)}(z)^\alpha} \mathrm{d} V^{(2)}(z)p^{(2)}(y)\mathrm{d} V^{(2)}\left(y\right)\\
=&\frac{p^{(2)}(x)q_{\alpha}(x)}{\operatorname{det}(\Lambda_x)}+\epsilon p^{(2)}(x)q_{\alpha}(x)W(x)-\epsilon\alpha q_{\alpha}(x)p^{(2)}(x)\bar{E}(x)\\
&+\frac{\epsilon \mu_{2,0}^{(1)}}{2d\operatorname{det}(\Lambda_x)}\bigg[q_{\alpha}(x)\sum^{d}_{i=1}\lambda_i\nabla_{E_iE_i}^{(2)^2}p^{(2)}(x)+p^{(2)}(x)\sum^{d}_{i=1}\lambda_i\nabla_{E_iE_i}^{(2)^2}q_{\alpha}(x)\\
&\hspace*{1.5cm}+2\left( \sum_{i=1}^d \lambda_i\left[\nabla^{(2)}_{E_i} p^{(2)}(x)\nabla^{(2)}_{E_i} q_{\alpha}(x)\right]\right)\bigg]\\
&+\frac{\epsilon \mu_{2,0}^{(2)}}{2d\operatorname{det}(\Lambda_x)}\bigg[q_{\alpha}(x)\Delta^{(2)}p^{(2)}(x)\bigg]+\mathcal{O}(\epsilon^{3/2})\,.
\end{aligned}
\label{eq:denom}
\end{equation}
Putting the above together, we obtain the claimed result
\begin{align*}
T_{\lan,\epsilon,\alpha}f(x)=f(x)\,&+\frac{\epsilon \mu_{2,0}^{(2)}}{2d}\Delta^{(2)}f(x)+\frac{\epsilon \mu_{2,0}^{(1)}}{2d}\sum^{d}_{i=1}\lambda_i\nabla^{(2)^2}_{E_iE_i}f(x)\\
&+\frac{\epsilon \mu_{2,0}^{(1)}}{d}\sum_{i=1}^d \lambda_i\left(\frac{\nabla^{(2)}_{E_i} p^{(2)}(x)}{p^{(2)}}+\frac{\nabla^{(2)}_{E_i} q_\alpha(x)}{q_\alpha(x)}\right)\nabla^{(2)}_{E_i} f(x)\\
&+\frac{\mu_{2,0}^{(2)}}{d}\frac{\nabla^{(2)} p^{(2)}(x)\cdot\nabla^{(2)} f(x)}{p^{(2)}(x)}+\mathcal{O}(\epsilon^{3/2})\,.
\end{align*}
\end{proof}

\subsection{Proof of Theorem \ref{thm:var}}
Recall the following results in \cite{shen2022}, which states the variance analysis of landmark kernels.
\begin{lemma}[Equation (46) in \cite{shen2022}]\label{lem:ker_var}
Take $\mathcal{X}=\left\{x_i\right\}_{i=1}^n$ and $\mathcal{Z}=\left\{z_k\right\}_{k=1}^m$, where $m=\left[n^\beta\right]$ for some $0<\beta \leq 1$ and $[r]$ is the nearest integer of $r \in \mathbb{R}$. Then with probability higher than $1-\mathcal{O}\left(1 / n^2\right)$, we have
\begin{equation*}
\frac{1}{n} \sum_{i=1}^n \epsilon^{-d / 2} K_\epsilon^{(2)}\left(z_k, x_i\right) K^{(2)}_\epsilon\left(x_i, z_l\right)=\epsilon^{-d/2}K^{(2)}_{\lan, \epsilon}\left(z_k, z_l\right)+O\left(\frac{\sqrt{\log (n)}}{n^{1 / 2} \epsilon^{d / 4}}\right) .
\end{equation*}
and
\begin{equation*}
\frac{1}{m} \sum_{k=1}^m \epsilon^{-d / 2} K_\epsilon^{(1)}\left(x_i, z_k\right) K^{(2)}_\epsilon\left(z_k, x_j\right)=\epsilon^{-d/2}K_{\lan, \epsilon,0}\left(z_k, z_l\right)+O\left(\frac{\sqrt{\log (n)}}{n^{\beta / 2} \epsilon^{d / 4}}\right) .
\end{equation*}
where $K^{(2)}_{\lan, \epsilon}\left(z_i, z_j\right)$ is defined in Definition \ref{def:lan2kernel} and  $K_{\lan, \epsilon,0}\left(z_i, z_j\right)$ is defined in Definition \ref{def:lanADkernel} with $\alpha=0$.
\end{lemma}

\begin{lemma}[Lemma C.1 in \cite{Singer2015}]\label{lem:deg_var}
Take $\mathcal{X}=\left\{x_i\right\}_{i=1}^n$ and $\mathcal{Z}=\left\{z_k\right\}_{k=1}^m$, where $m=\left[n^\beta\right]$ for some $0<\beta \leq 1$ and $[r]$ is the nearest integer of $r \in \mathbb{R}$. Then with probability higher than $1-\mathcal{O}\left(1 / n^2\right)$, we have
\begin{equation*}
\frac{1}{m} \sum_{l=1}^m \epsilon^{-d} K^{(2)}_{\lan,\epsilon}(z_k,z_l)=\epsilon^{-d}d^{(2)}_{\lan, \epsilon}\left(z_k\right)+O\left(\frac{\sqrt{\log (n)}}{n^{\beta / 2} \epsilon^{d / 4}}\right) .
\end{equation*}
and
\begin{equation*}
\frac{1}{n} \sum_{j=1}^n \epsilon^{-d} K_{\lan,\epsilon,0}(x_i,x_j)=\epsilon^{-d}d_{\lan, \epsilon,0}\left(x_i\right)+O\left(\frac{\sqrt{\log (n)}}{n^{1 / 2} \epsilon^{d / 4}}\right) .
\end{equation*}
where $d^{(2)}_{\lan, \epsilon}\left(z_l\right)$ and $d_{\lan,\epsilon,0}(x_i)$ are defined in Definition \ref{def:lan2kernel} and \ref{def:lanADkernel}, respectively.
\end{lemma}

The following lemma shows the large deviation bound of original diffusion operator.
\begin{lemma}[\cite{singer2006}]\label{lem:diff_var}
Take $\mathcal{X}=\left\{x_i\right\}_{i=1}^n$ and $\mathcal{Z}=\left\{z_k\right\}_{k=1}^m$, where $m=\left[n^\beta\right]$ for some $0<\beta \leq 1$ and $[r]$ is the nearest integer of $r \in \mathbb{R}$. Let $\epsilon=\epsilon(n)$ so that $\frac{\sqrt{\log n}}{n^{\beta / 2} \epsilon^{d / 2+1 / 2}} \rightarrow 0$ and $\epsilon \rightarrow 0$ when $n \rightarrow \infty$. Then with probability higher than $1-\mathcal{O}\left(1 / n^2\right)$, we have
\begin{equation*}
\frac{\frac{1}{n}\sum_{j=1}^n K_{\lan, \epsilon,\alpha}(x_i, x_j) (f(x_i)-f(x_j)) }{\frac{1}{n}\sum_{j=1}^n K_{\lan, \epsilon,\alpha}(x_i, x_j) }=f(x_i)-T_{\lan,\epsilon,\alpha}f(x_i)+\mathcal{O}\left(\frac{\sqrt{\log n}}{n^{1/2}\epsilon^{d/4-1/2}}\right)
\end{equation*}
\end{lemma}

Before we prove Theorem \ref{thm:var}, we define some functions for convenience.
\begin{align*}
\hat{K}^{(2)}_{\lan,\epsilon,n}(z_k,z_l)&:=\frac{1}{n} \sum_{i=1}^n \epsilon^{-d / 2} K_\epsilon^{(2)}\left(z_k, x_i\right) K^{(2)}_\epsilon\left(x_i, z_l\right),\\
\hat{d}^{(2)}_{\lan,\epsilon,m}(z_k)&:=\frac{1}{m} \sum_{l=1}^m \epsilon^{-d / 2} \hat{K}^{(2)}_{\lan,\epsilon,m}(z_k,z_l),\\
\hat{K}_{\lan,\epsilon,\alpha,m}(x_i,x_j)&:=\frac{1}{m} \sum_{k=1}^m \epsilon^{-d / 2} K_\epsilon^{(1)}\left(x_i, z_k\right) K^{(2)}_\epsilon\left(z_k, x_j\right)\hat{d}^{(2)}_{\lan,\epsilon,m}(z_k)^{-\alpha},\\
\hat{d}_{\lan,\epsilon,\alpha,n}(x_i)&:=\frac{1}{n} \sum_{j=1}^n \epsilon^{-d / 2} \hat{K}_{\lan,\epsilon,\alpha,m}(x_i,x_j),
\end{align*}
and 
\begin{align*}
\tilde{d}^{(2)}_{\lan,\epsilon,m}(z_k)&:=\frac{1}{m} \sum_{l=1}^m \epsilon^{-d} K^{(2)}_{\lan,\epsilon}(z_k,z_l),\\
\tilde{K}_{\lan,\epsilon,\alpha,m}(x_i,x_j)&:=\frac{1}{m} \sum_{k=1}^m \epsilon^{-d / 2+d\alpha} K_\epsilon^{(1)}\left(x_i, z_k\right) K^{(2)}_\epsilon\left(z_k, x_j\right)d^{(2)}_{\lan,\epsilon}(z_k)^{-\alpha}\,.
\end{align*}

Now, it is sufficient to prove the Theorem \ref{thm:var}.
\begin{proof}
Fix $x_i\in\M$ and $\alpha\in [0,1]$. For convenience, define
\begin{equation*}
[L\boldsymbol{f}](i):=\frac{1}{\epsilon}\left[\left(I_n-\left(\mathbf{D}^{(1)}_\alpha\right)^{-1} \mathbf{W}^{(1)}_\alpha\mathbf{M}^{(2)}_\alpha\right) \boldsymbol{f}\right](i)
\end{equation*}
for all $i=1,2, \ldots, n$. By definition, we have
\begin{align*}
&\epsilon[L\boldsymbol{f}](i)=\frac{\frac{1}{n}\sum_{j=1}^n \hat{K}_{\lan, \epsilon,\alpha,m}(x_i, x_j) (f(x_i)-f(x_j)) }{\frac{1}{n}\sum_{j=1}^n \hat{K}_{\lan, \epsilon,\alpha,m}(x_i, x_j) }\\
=&\frac{\frac{1}{n}\sum_{j=1}^n \left(\frac{1}{m}\sum_{k=1}^mK^{(1)}_\epsilon(x_i,z_k)K^{(2)}_\epsilon(z_k,x_j)\hat{d}^{(2)}_{\lan,\epsilon,m}(z_k)^{-\alpha}\right) (f(x_i)-f(x_j)) }{\frac{1}{n}\sum_{j=1}^n \left(\frac{1}{m}\sum_{k=1}^mK^{(1)}_\epsilon(x_i,z_k)K^{(2)}_\epsilon(z_k,x_j)\hat{d}^{(2)}_{\lan,\epsilon,m}(z_k)^{-\alpha}\right)}\\
=&\frac{\frac{1}{n}\sum_{j=1}^n \left(\frac{1}{m}\sum_{k=1}^mK^{(1)}_\epsilon(x_i,z_k)K^{(2)}_\epsilon(z_k,x_j)\left(\frac{1}{m}\sum_{l=1}^m\hat{K}^{(2)}_{\lan,\epsilon,n}(z_k,z_l)\right)^{-\alpha}\right) (f(x_i)-f(x_j)) }{\frac{1}{n}\sum_{j=1}^n \left(\frac{1}{m}\sum_{k=1}^mK^{(1)}_\epsilon(x_i,z_k)K^{(2)}_\epsilon(z_k,x_j)\left(\frac{1}{m}\sum_{l=1}^m\hat{K}^{(2)}_{\lan,\epsilon,n}(z_k,z_l)\right)^{-\alpha}\right)}
\end{align*}

First, we control the term $\hat{K}^{(2)}_{\lan,\epsilon,m}$. For convenience, we rewrite
\begin{align*}
\hat{K}^{(2)}_{\lan,\epsilon,n}(z_k,z_l)&=\epsilon^{-d/2}K^{(2)}_{\lan,\epsilon}(z_k,z_l)-\left(\epsilon^{-d/2}K^{(2)}_{\lan,\epsilon}(z_k,z_l)-\hat{K}^{(2)}_{\lan,\epsilon,n}(z_k,z_l)\right)\\
&:=\epsilon^{-d/2}K^{(2)}_{\lan,\epsilon}(z_k,z_l)-D(z_k,z_l)\,.
\end{align*}
Calculate directly,
\begin{equation*}
\epsilon[L\boldsymbol{f}](i)=J_1+J_2+J_3
\end{equation*}
where
\begin{equation*}
J_1:=\frac{\frac{1}{n}\sum_{j=1}^n \left(\frac{1}{m}\sum_{k=1}^mK^{(1)}_\epsilon(x_i,z_k)K^{(2)}_\epsilon(z_k,x_j)\left(\frac{1}{m}\sum_{l=1}^mK^{(2)}_{\lan,\epsilon,n}(z_k,z_l)\right)^{-\alpha}\right) (f(x_i)-f(x_j)) }{\frac{1}{n}\sum_{j=1}^n \left(\frac{1}{m}\sum_{k=1}^mK^{(1)}_\epsilon(x_i,z_k)K^{(2)}_\epsilon(z_k,x_j)\left(\frac{1}{m}\sum_{l=1}^m K^{(2)}_{\lan,\epsilon,n}(z_k,z_l)\right)^{-\alpha}\right)}\,,
\end{equation*}
\begin{equation*}
J_2:=\frac{\frac{1}{n}\sum_{j=1}^n \left(\frac{1}{m}\sum_{k=1}^mK^{(1)}_\epsilon(x_i,z_k)K^{(2)}_\epsilon(z_k,x_j)\left(\frac{1}{m}\sum_{l=1}^m D(z_k,z_l)\right)^{-\alpha}\right) (f(x_i)-f(x_j)) }{\frac{1}{n}\sum_{j=1}^n \left(\frac{1}{m}\sum_{k=1}^mK^{(1)}_\epsilon(x_i,z_k)K^{(2)}_\epsilon(z_k,x_j)\left(\frac{1}{m}\sum_{l=1}^m \epsilon^{-d/2}K^{(2)}_{\lan,\epsilon,\alpha}(z_k,z_l)\right)^{-\alpha}\right)}
\end{equation*}
and
\begin{align*}
J_3:=&\frac{1}{n}\sum_{j=1}^n \left(\frac{1}{m}\sum_{k=1}^mK^{(1)}_\epsilon(x_i,z_k)K^{(2)}_\epsilon(z_k,x_j)\left(\frac{1}{m}\sum_{l=1}^m\hat{K}^{(2)}_{\lan,\epsilon,n}(z_k,z_l)\right)^{-\alpha}\right) (f(x_i)-f(x_j))\\
&\hspace*{0.5cm}\times\bigg[\left(\frac{1}{n}\sum_{j=1}^n \left(\frac{1}{m}\sum_{k=1}^mK^{(1)}_\epsilon(x_i,z_k)K^{(2)}_\epsilon(z_k,x_j)\left(\frac{1}{m}\sum_{l=1}^m \epsilon^{-d/2}K^{(2)}_{\lan,\epsilon,\alpha}(z_k,z_l)\right)^{-\alpha}\right)\right)^{-1}\\
&\hspace*{1.cm}-\left(\frac{1}{n}\sum_{j=1}^n \left(\frac{1}{m}\sum_{k=1}^mK^{(1)}_\epsilon(x_i,z_k)K^{(2)}_\epsilon(z_k,x_j)\left(\frac{1}{m}\sum_{l=1}^m \hat{K}^{(2)}_{\lan,\epsilon,\alpha}(z_k,z_l)\right)^{-\alpha}\right)\right)^{-1}\bigg]
\end{align*}

By Lemma \ref{lem:ker_var}, with probability $1-\mathcal{O}(n^{-2})$, we have
\begin{equation}\label{eq:event_1}
\hat{K}^{(2)}_{\lan,\epsilon,n}(z_k,z_l)=\epsilon^{-d/2}K^{(2)}_{\lan,\epsilon}(z_k,z_l)+\mathcal{O}\left(\frac{\sqrt{\log (n)}}{n^{1 / 2} \epsilon^{d / 4}}\right)
\end{equation}
Denote by $\Omega_1$ the event space that Equation (\ref{eq:event_1}) holds. Under $\Omega_1$, we can easily control $J_2$ and $J_3$, and hence we have
\begin{align*}
\epsilon[L\boldsymbol{f}](i)=&\frac{\frac{1}{n}\sum_{j=1}^n \left(\frac{1}{m}\sum_{k=1}^mK^{(1)}_\epsilon(x_i,z_k)K^{(2)}_\epsilon(z_k,x_j)\tilde{d}^{(2)}_{\lan,\epsilon,m}(z_k)^{-\alpha}\right) (f(x_i)-f(x_j)) }{\frac{1}{n}\sum_{j=1}^n \left(\frac{1}{m}\sum_{k=1}^mK^{(1)}_\epsilon(x_i,z_k)K^{(2)}_\epsilon(z_k,x_j)\tilde{d}^{(2)}_{\lan,\epsilon,m}(z_k)^{-\alpha}\right)}\\
&\hspace*{1cm}+\mathcal{O}\left(\alpha \frac{\sqrt{\log (n)}}{n^{1/ 2} \epsilon^{d / 4}}\right)
\end{align*}
where the implied constant depends on $\|f\|_\infty$.

Second, again, calculate directly,
\begin{align*}
&\frac{\frac{1}{n}\sum_{j=1}^n \left(\frac{1}{m}\sum_{k=1}^mK^{(1)}_\epsilon(x_i,z_k)K^{(2)}_\epsilon(z_k,x_j)\tilde{d}^{(2)}_{\lan,\epsilon,m}(z_k)^{-\alpha}\right) (f(x_i)-f(x_j)) }{\frac{1}{n}\sum_{j=1}^n \left(\frac{1}{m}\sum_{k=1}^mK^{(1)}_\epsilon(x_i,z_k)K^{(2)}_\epsilon(z_k,x_j)\tilde{d}^{(2)}_{\lan,\epsilon,m}(z_k)^{-\alpha}\right)}\\
=&J_4+J_5+J_6
\end{align*}
where
\begin{equation*}
J_4:=\frac{\frac{1}{n}\sum_{j=1}^n \left(\frac{1}{m}\sum_{k=1}^mK^{(1)}_\epsilon(x_i,z_k)K^{(2)}_\epsilon(z_k,x_j)d^{(2)}_{\lan,\epsilon}(z_k)^{-\alpha}\right) (f(x_i)-f(x_j)) }{\frac{1}{n}\sum_{j=1}^n \left(\frac{1}{m}\sum_{k=1}^mK^{(1)}_\epsilon(x_i,z_k)K^{(2)}_\epsilon(z_k,x_j)d^{(2)}_{\lan,\epsilon}(z_k)^{-\alpha}\right)}
\end{equation*}
\begin{equation*}
J_5:=\frac{\frac{1}{n}\sum_{j=1}^n \left(\frac{1}{m}\sum_{k=1}^mK^{(1)}_\epsilon(x_i,z_k)K^{(2)}_\epsilon(z_k,x_j)(\tilde{d}^{(2)}_{\lan,\epsilon,m}(z_k)^{-\alpha}-\epsilon^{d\alpha}d^{(2)}_{\lan,\epsilon}(z_k)^{-\alpha}\right) (f(x_i)-f(x_j)) }{\frac{1}{n}\sum_{j=1}^n \left(\frac{1}{m}\sum_{k=1}^mK^{(1)}_\epsilon(x_i,z_k)K^{(2)}_\epsilon(z_k,x_j)\epsilon^{d\alpha}d^{(2)}_{\lan,\epsilon}(z_k)^{-\alpha}\right)}
\end{equation*}
and
\begin{align*}
J_6:=&\frac{1}{n}\sum_{j=1}^n \left(\frac{1}{m}\sum_{k=1}^mK^{(1)}_\epsilon(x_i,z_k)K^{(2)}_\epsilon(z_k,x_j)\tilde{d}^{(2)}_{\lan,\epsilon,m}(z_k)^{-\alpha}\right) (f(x_i)-f(x_j))\\
&\hspace*{5mm}\times\bigg[\left(\frac{1}{n}\sum_{j=1}^n \left(\frac{1}{m}\sum_{k=1}^mK^{(1)}_\epsilon(x_i,z_k)K^{(2)}_\epsilon(z_k,x_j)\epsilon^{d\alpha}d^{(2)}_{\lan,\epsilon}(z_k)^{-\alpha}\right)\right)^{-1}\\
&\hspace*{10mm}-\left(\frac{1}{n}\sum_{j=1}^n \left(\frac{1}{m}\sum_{k=1}^mK^{(1)}_\epsilon(x_i,z_k)K^{(2)}_\epsilon(z_k,x_j)\tilde{d}^{(2)}_{\lan,\epsilon,m}(z_k)^{-\alpha}\right)\right)^{-1}\bigg]
\end{align*}
By Lemma \ref{lem:deg_var}, with probability $1-\mathcal{O}(n^{-2})$, we have
\begin{equation}\label{eq:event_2}
\tilde{d}^{(2)}_{\lan,\epsilon,m}=\epsilon^{-d}d^{(2)}_{\lan, \epsilon}\left(z_k\right)+\mathcal{O}\left(\frac{\sqrt{\log (n)}}{n^{\beta / 2} \epsilon^{d / 4}}\right)
\end{equation}
Denote by $\Omega_2$ the event space that Equation (\ref{eq:event_2}) holds. Thus, putting the above together, under $\Omega_1\cap\Omega_2$, we can control $J_5$ and $J_6$ and hence we have
\begin{align*}
\epsilon[L\boldsymbol{f}](i)=&\frac{\frac{1}{n}\sum_{j=1}^n \left(\frac{1}{m}\sum_{k=1}^mK^{(1)}_\epsilon(x_i,z_k)K^{(2)}_\epsilon(z_k,x_j)d^{(2)}_{\lan,\epsilon}(z_k)^{-\alpha}\right) (f(x_i)-f(x_j)) }{\frac{1}{n}\sum_{j=1}^n \left(\frac{1}{m}\sum_{k=1}^mK^{(1)}_\epsilon(x_i,z_k)K^{(2)}_\epsilon(z_k,x_j)d^{(2)}_{\lan,\epsilon}(z_k)^{-\alpha}\right)}\\
&\hspace*{1cm}+\mathcal{O}\left(\alpha \frac{\sqrt{\log (n)}}{n^{\beta / 2} \epsilon^{d / 4}}\right)
\end{align*}
where $\frac{\sqrt{\log (n)}}{n^{\beta / 2} \epsilon^{d / 4}}\geq \frac{\sqrt{\log (n)}}{n^{1 / 2} \epsilon^{d / 4}}$ and the implied constant depends on $\|f\|_\infty$.

Third, again, calculate directly,
\begin{align}
&\frac{\frac{1}{n}\sum_{j=1}^n \tilde{K}_{\lan,\epsilon,\alpha,m}(x_i,x_j) (f(x_i)-f(x_j)) }{\frac{1}{n}\sum_{j=1}^n \tilde{K}_{\lan,\epsilon,\alpha,m}(x_i,x_j)}\label{proof expansion to 3 terms example}\\
=&\frac{\frac{1}{n}\sum_{j=1}^n K_{\lan,\epsilon,\alpha}(x_i,x_j) (f(x_i)-f(x_j)) }{\frac{1}{n}\sum_{j=1}^n K_{\lan,\epsilon,\alpha}(x_i,x_j)}\nonumber\\
&+\frac{\frac{1}{n}\sum_{j=1}^n (\tilde{K}_{\lan,\epsilon,\alpha,m}(x_i,x_j)-\epsilon^{-d/2+d\alpha}K_{\lan,\epsilon,\alpha}(x_i,x_j)) (f(x_i)-f(x_j)) }{\frac{1}{n}\sum_{j=1}^n \epsilon^{-d/2+d\alpha}K_{\lan,\epsilon,\alpha}(x_i,x_j)}\nonumber\\
&+\frac{1}{n}\sum_{j=1}^n \tilde{K}_{\lan,\epsilon,\alpha,m}(x_i,x_j) (f(x_i)-f(x_j))\nonumber\\
&\hspace*{1cm}\times\bigg[\left(\frac{1}{n}\sum_{j=1}^n \epsilon^{-d/2+d\alpha}K_{\lan,\epsilon,\alpha}(x_i,x_j)\right)^{-1}-\left(\frac{1}{n}\sum_{j=1}^n \tilde{K}_{\lan,\epsilon,\alpha,m}(x_i,x_j)\right)^{-1}\bigg]\,.\nonumber
\end{align}

We need to control the term  $|\tilde{K}_{\lan,\epsilon,\alpha,m}(x_i,x_j)-\epsilon^{-d/2+d\alpha}K_{\lan,\epsilon,\alpha}(x_i,x_j)|$, which is generalized the proof of Lemma \ref{lem:ker_var}. Let a random variable 
\begin{equation*}
F:=\epsilon^{-d/2+d\alpha}K^{(1)}_\epsilon(x_i,Z)K^{(2)}_\epsilon(Z,x_j)d^{(2)}_{\lan,\epsilon}(Z)^{-\alpha}
\end{equation*}
where $Z\sim p^{(2)}_\Z$. Denote by $F_k$ one realization of $F$ when the realization of random variable $Z$ is $z_k$. Now, we need to find the convergence rate of
\begin{equation*}
\frac{1}{m}\sum_{k=1}^mF_k\rightarrow\mathbb{E}[F]\,.
\end{equation*}
By Lemma \ref{lem:refkernel}, let $y=\exp_{x_i}^{(2)} v$ where $v\in T_{x_i}\M$ and we have
\begin{align*}
\mathbb{E}[F]&=\epsilon^{-d/2+d\alpha}\tilde{K}_{\lan,\epsilon,\alpha,m}(x_i,x_j)\\
&=\epsilon^{-d/2+d\alpha}\int_\M K^{(1)}_\epsilon(x_i,z)K^{(2)}_\epsilon(z,x_j)d^{(2)}_{\lan,\epsilon}(z)^{-\alpha}p^{(2)}_\Z(z)dV(z)\\
&=p^{(2)}(z)^{-\alpha}p^{(2)}_\Z(z)^{1-\alpha}A_0(v)+\mathcal{O}(\epsilon)
\end{align*}
Furthermore, by Equation (\ref{eq:lanker_decay}), we know $A_0(v)$ decay fast as $\|v\|_{g^{(2)}}$ small. On the other hand, we have
\begin{equation*}
    \mathbb{E}[F^2]=\epsilon^{-d+2d\alpha}\int_\M K^{(1)}_\epsilon(x_i,z)^2 K^{(2)}_\epsilon(z,x_j)^2 d^{(2)}_{\lan,\epsilon}(z)^{-2\alpha}p^{(2)}_\Z(z)dV(z)\asymp \epsilon^{-d/2}\,.
\end{equation*}
Hence, we have $\operatorname{Var}(F)\asymp \epsilon^{-d/2}$ and $F_k\asymp \epsilon^{-d/2}$. By Berstein inequality, we know
\begin{equation*}
\mathbb{P}\left(\frac{1}{m} \sum_{k=1}^m F_k-\mathbb{E}(F)>t\right) \leq \exp \left(-\frac{m t^2}{2 \epsilon^{-d / 2}+\frac{2}{3} \epsilon^{-d / 2} t}\right)\,.
\end{equation*}
Since we ask $\frac{t}{\mathbb{E}[F]} \rightarrow 0$ when $m \rightarrow \infty$, the exponent becomes
\begin{equation*}
\frac{m t^2}{2 \epsilon^{-d / 2}+\frac{2}{3} \epsilon^{-d / 2} t} \geq \frac{m t^2}{3 \epsilon^{-d / 2}} .
\end{equation*}

Then if we choose $m$ such that $\frac{m t^2}{3 \epsilon^{-d / 2}}=4 \log (n)$, we have
\begin{equation*}
t\asymp\frac{\sqrt{\log n}}{n^{\beta / 2} \epsilon^{d / 4}},
\end{equation*}
which goes to 0 by our assumption $\frac{\sqrt{\log n}}{n^{\beta / 2} \epsilon^{d /4}} \rightarrow 0$.
By union bound and get for all pairs $(x_i,x_j)$, $i,j=1,\cdots,n$, with probability $1-\mathcal{O}(n^{-2})$
\begin{equation}\label{eq:event_3}
    \tilde{K}_{\lan,\epsilon,\alpha,m}(x_i,x_j)=\epsilon^{-d/2+d\alpha}K_{\lan,\epsilon,\alpha}(x_i,x_j)+\mathcal{O}\left(\frac{\sqrt{\log n}}{n^{\beta / 2} \epsilon^{d /4}}\right)
\end{equation}
Denote by $\Omega_3$ the event space that Equation (\ref{eq:event_3}) holds. Thus, putting the above together, under $\Omega_1\cap\Omega_2\cap\Omega_3$, we have
\begin{align*}
\epsilon[L\boldsymbol{f}](i)=&\frac{\frac{1}{n}\sum_{j=1}^n K_{\lan,\epsilon,\alpha}(x_i,x_j) (f(x_i)-f(x_j)) }{\frac{1}{n}\sum_{j=1}^n K_{\lan,\epsilon,\alpha}(x_i,x_j)}+\mathcal{O}\left( \frac{\sqrt{\log (n)}}{n^{\beta / 2} \epsilon^{d / 4}}\right)
\end{align*}
where the implied constant depends on $\|f\|_\infty$.

Finally, we need to apply Lemma \ref{lem:diff_var} to get the convergence rate. Let random variables
\begin{align*}
&G:=\epsilon^{-d+d\alpha}K_{\lan,\epsilon,\alpha}(x_i,Y)(f(x_i)-f(Y))\\
&H:=\epsilon^{-d+d\alpha}K_{\lan,\epsilon,\alpha}(x_i,Y)
\end{align*}
where $Y\sim p^{(2)}$.

By Equation (\ref{eq:numer}), we have
\begin{align*}
\mathbb{E}[G] =&\frac{\epsilon \mu_{2,0}^{(1)}}{2d\operatorname{det}(\Lambda_x)}2q_{\alpha}(x_i)\left( \sum_{i=1}^d \lambda_i\left[\nabla^{(2)}_{E_i} p^{(2)}(x)\nabla^{(2)}_{E_i} (f(x_i)-f(x))|_{x=x_i}\right]\right)\\
&+\frac{\epsilon \mu_{2,0}^{(1)}}{2d\operatorname{det}(\Lambda_x)}2p^{(2)}(x_i)\left( \sum_{i=1}^d \lambda_i\left[\nabla^{(2)}_{E_i} q_{\alpha}(x)\nabla^{(2)}_{E_i} (f(x_i)-f(x))|_{x=x_i}\right]\right)\\
&+\frac{\epsilon \mu_{2,0}^{(2)}}{2d\operatorname{det}(\Lambda_x)}2q_{\alpha}(x_i)\nabla^{(2)} p^{(2)}(x)\cdot\nabla^{(2)}(f(x_i)-f(x))|_{x=x_i}\\
&+\frac{\epsilon \mu_{2,0}^{(2)}}{2d\operatorname{det}(\Lambda_x)}q_{\alpha}(x_i)p^{(2)}(x_i)\Delta^{(2)}(f(x_i)-f(x))|_{x=x_i}\\
&+\frac{\epsilon \mu_{2,0}^{(1)}}{2d\operatorname{det}(\Lambda_x)}q_{\alpha}(x_i)p^{(2)}(x_i)\sum^{d}_{i=1}\lambda_i\nabla^{(2)^2}_{E_iE_i}(f(x_i)-f(x))|_{x=x_i}+\mathcal{O}(\epsilon^{3/2})
\end{align*}
and
\begin{align*}
\mathbb{E}[G^2] =&\frac{\epsilon^{1-d/2} \mu_{2,0,2}^{(2)}}{2d\operatorname{det}(\Lambda_x)}q_{\alpha}(x_i)p^{(2)}(x_i)\Delta^{(2)}(f(x_i)-f(x))^2|_{x=x_i}\\
&+\frac{\epsilon^{1-d/2} \mu_{2,0,2}^{(1)}}{2d\operatorname{det}(\Lambda_x)}q_{\alpha}(x_i)p^{(2)}(x_i)\sum^{d}_{i=1}\lambda_i\nabla^{(2)^2}_{E_iE_i}(f(x_i)-f(x))^2|_{x=x_i}+\mathcal{O}(\epsilon^{3/2})\\
=&\frac{\epsilon^{1-d/2} \mu_{2,0,2}^{(2)}}{2d\operatorname{det}(\Lambda_x)}q_{\alpha}(x_i)p^{(2)}(x_i)\|\nabla^{(2)}(f(x_i)-f(x))\|^2|_{x=x_i}\\
&+\frac{\epsilon^{1-d/2} \mu_{2,0,2}^{(1)}}{2d\operatorname{det}(\Lambda_x)}q_{\alpha}(x_i)p^{(2)}(x_i)\sum^{d}_{i=1}\lambda_i|\nabla^{(2)}_{E_i}(f(x_i)-f(x))|^2|_{x=x_i}+\mathcal{O}(\epsilon^{3/2})
\end{align*}
where $\mu^{(\ell)}_{2,0,2}=\int_{\mathbb{R}^d}\|x\|^2 \tilde{K}^{(\ell)}(\|x\|)^2 d x$. By Equation (\ref{eq:denom}), we have
\begin{equation*}
\mathbb{E}[H]=\frac{p^{(2)}(x)q_{\alpha}(x)}{\operatorname{det}(\Lambda_x)}+\mathcal{O}(\epsilon)
\end{equation*}
and
\begin{equation*}
\mathbb{E}[H^2]=\epsilon^{-d/2}\mu^{(1)}_{0,0,2}\mu^{(2)}_{0,0,2}\frac{p^{(2)}(x)q_{\alpha}(x)}{\operatorname{det}(\Lambda_x)}+\mathcal{O}(\epsilon^{-d/2+1})
\end{equation*}
where $\mu^{(\ell)}_{0,0,2}=\int_{\mathbb{R}^d} \tilde{K}^{(\ell)}(\|x\|)^2 d x$. Now, it is sufficient to apply Lemma \ref{lem:diff_var} to conclude the proof and further look closer to what the implied constant depends. Thus, with probability $1-\mathcal{O}(n^{-2})$, we have
\begin{align}\label{eq:event_4}
&\frac{\frac{1}{n}\sum_{j=1}^n K_{\lan, \epsilon,\alpha}(x_i, x_j) (f(x_i)-f(x_j)) }{\frac{1}{n}\sum_{j=1}^n K_{\lan, \epsilon,\alpha}(x_i, x_j) }\\
=&\,f(x_i)-T_{\lan,\epsilon,\alpha}f(x_i)+\mathcal{O}\left(\frac{\sqrt{\log n}}{n^{1/2}\epsilon^{d/4-1/2}}\right)\nonumber
\end{align}
where the implied constant depends on
\begin{equation*}
\|\nabla^{(2)}(f(x_i)-f(x))\|^2|_{x=x_i}+\sum^{d}_{i=1}\lambda_i|\nabla^{(2)}_{E_i}(f(x_i)-f(x))|^2|_{x=x_i}\,.
\end{equation*}
Denote by $\Omega_4$ the event space that Equation (\ref{eq:event_4}) holds. Thus, putting the above together, under $\Omega_1\cap\Omega_2\cap\Omega_3\cap\Omega_4$,
\begin{equation}\label{eq:main_var}
\frac{1}{\epsilon}[L\boldsymbol{f}](i)=\frac{f(x_i)-T_{\lan,\epsilon,\alpha}f(x_i)}{\epsilon}+\mathcal{O}\left(\frac{\sqrt{\log n}}{n^{\beta/2}\epsilon^{d/4+1}}\right)
\end{equation}
where $\frac{\sqrt{\log n}}{n^{1/2}\epsilon^{d/4-1/2}}\leq \frac{\sqrt{\log n}}{n^{\beta/2}\epsilon^{d/4}}$. That is, with probability $1-\mathcal{O}(n^{-2})$, Equation (\ref{eq:main_var}) holds.
\end{proof}

Before closing this section, we provide a numerical exploration of the obtained convergence rate in the variance analysis by studying the relation between $n$ and $\epsilon$. We conduct two experiments following \cite{singer2006} to examine the variance term in Theorem \ref{thm:var}. First, consider the canonical $\mathbb{S}^1$ and the scaled circle $S$, defined as
\begin{align*}
\mathbb{S}^1 &= \{(\cos \theta, \, \sin \theta)\} \subset\mathbb{R}^2, \\
S_1 &= \{(1.5\cos \theta, \, 1.5\sin \theta)\} \subset\mathbb{R}^2\,,
\end{align*}
where $\theta \in [-\pi, \pi)$. We randomly select 3000 points $\{(s_i,r_i):s_i\in S_1, r_i\in\mathbb{S}^1\}$ uniformly and we select 1500 landmark points followed the non-uniform p.d.f. $p_\Z^{(2)}(\theta)=\frac{58}{50}[0.48\cos\theta+0.52]$. Consider the function $f(x,y)=x$ and evaluate at $(\frac{1}{\sqrt{2}},\frac{1}{\sqrt{2}})$. We have $\Delta^{(2)}f\vert_{(\frac{1}{\sqrt{2}},\frac{1}{\sqrt{2}})}=-\frac{1}{\sqrt{2}}$, $\lambda_1 \nabla^{(2)^2}_{E_1,E_1}f\vert_{(\frac{1}{\sqrt{2}},\frac{1}{\sqrt{2}})}=\frac{-2}{3\sqrt{2}}$ and $\lambda_1\frac{\nabla_{E_1}^{(2)} p^{(2)^{1-\alpha}}_\Z}{p^{(2)^{1-\alpha}}_\Z} \nabla_{E_1}^{(2)} f\vert_{(\frac{1}{\sqrt{2}},\frac{1}{\sqrt{2}})}\neq 0$. Then, for fix $x_i$, we conduct $T=50$ independent trials and compute the deviation 
\begin{equation*}
\sqrt{\frac{1}{T}\sum^{T}_{t=1}\left|\frac{1}{\epsilon}\left[\left(I_n-\left(\mathbf{D}^{(1)}_{\alpha,t}\right)^{-1} \mathbf{W}^{(1)}_{\alpha,t}\mathbf{M}^{(2)}_{\alpha,t}\right) \boldsymbol{f}\right](i)-\frac{f\left(x_i\right)-T_{\lan,\epsilon,\alpha}f(x_i)}{\epsilon}\right|^2}\,,
\end{equation*}
where $\mathbf{D}^{(1)}_{\alpha,t}$, $\mathbf{D}^{(1)}_{\alpha,t}$ and $\mathbf{M}^{(2)}_{\alpha,t}$ are constructed in the $t$-th trial. The relationship between the deviation and $\epsilon$ is shown in the left subfigure of Figure \ref{fig:var}. According to Theorem \ref{thm:var}, we expect the convergence rate to be $\epsilon^{5/4}$. However, a linear fit yields slopes $0.75, 0.71, 0.73, 0.71$ and $0.73$ when $\alpha = 0, 0.25, 0.5, 0.75, 1.0$ respectively. This numerical result is faster than our analysis.

Second, consider the canonical $\mathbb{S}^2$ and the scaled sphere $S_2$, defined as
\begin{align*}
\mathbb{S}^2 &= \{(\sin\theta\cos \phi, \, \sin \theta\sin\phi,\, \cos\theta)\} \subset\mathbb{R}^3, \\
S_2 &= \{(0.8\sin\theta\cos \phi, \, 0.8\sin \theta\sin\phi,\, 0.8\cos\theta)\} \subset\mathbb{R}^3\,,
\end{align*}
where $\phi\in[-2\pi,2\pi)$ and $\theta \in [-\pi, \pi)$. We randomly select 3000 points $\{(s_i,r_i):s_i\in S_2, r_i\in\mathbb{S}^2\}$ uniformly and we select 2000 landmark points uniformly. Consider the function $f(x,y,z)=x+z$ and evaluate at $(0,0,1)$. The relationship between the deviation and $\epsilon$ is shown in the right subfigure of Figure \ref{fig:var}. According to Theorem \ref{thm:var}, we expect the convergence rate to be $\epsilon^{3/2}$. However, a linear fit yields slopes $0.92, 0.94, 0.94, 0.96$ and $0.96$ for $\alpha = 0, 0.25, 0.5, 0.75, 1.0$, respectively, which is also faster than our analysis.

\begin{figure}[hbt!]
\begin{center}
\includegraphics[trim=50 0 50 0, width=0.9\textwidth]{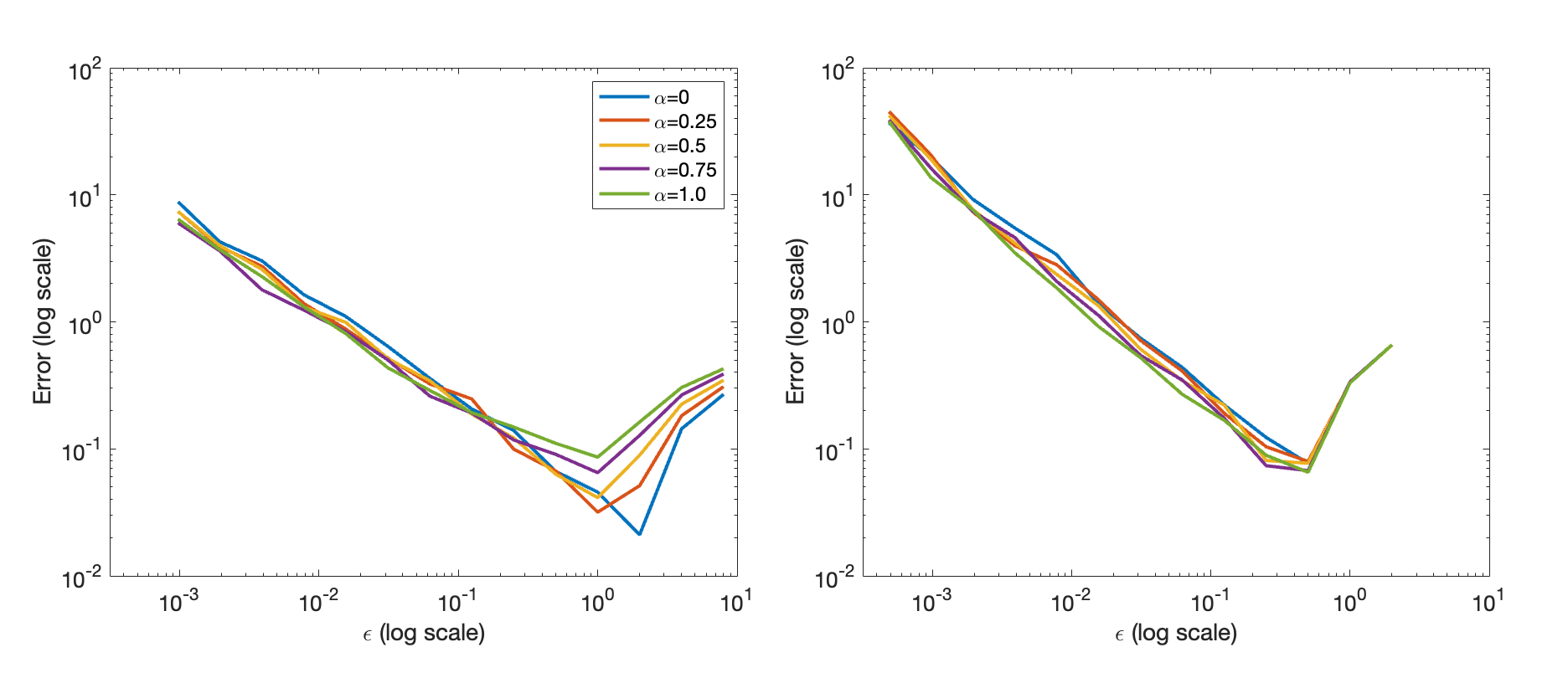}
\end{center}
\caption{Left: The deviation between the discrete and continuous $\alpha$-LAD operator on $\mathbb{S}^1$ and $S_1$. Right: The deviation between the discrete and continuous $\alpha$-LAD operator on $\mathbb{S}^2$ and $S_2$.}
\label{fig:var}
\end{figure}

According to Figure \ref{fig:var}, we anticipate that the convergence rate should be $\epsilon^{d/4+1/2}$ but it does not match the theoretical convergence rate we derive. The main reason is that we control the stochastic fluctuation by controlling the datasets and landmarks separately. For example, in \eqref{proof expansion to 3 terms example}, the control is divided into three terms, where the last two terms contribute to this slower theoretical convergence rate. We hypothesize that by properly handling the dependence induced by the diffusion between datasets and landmarks, we could improve the theoretical convergence rate.

\section{Conclusion}
We introduce a novel algorithm called Landmark Alternating Diffusion (LAD) designed to enhance the performance of the commonly employed kernel sensor fusion method, Alternating Diffusion (AD). Furthermore, we integrate an $\alpha$-normalization parameter to mitigate the influence of landmark sampling schemes. Apart from significantly improved computational efficiency, LAD exhibits asymptotic behavior akin to that of AD within a manifold framework.

\nocite{*}
\bibliographystyle{siamplain}
\bibliography{main}

\end{document}